\documentclass{article} %
\usepackage{iclr2024_conference,times}

\usepackage{microtype}
\usepackage{graphicx}
\usepackage{subfigure}
\usepackage{booktabs} %
\usepackage{xspace}
\usepackage{titletoc}
\usepackage{enumitem}
\usepackage{xfrac}
\usepackage{wrapfig}
\usepackage{algorithm,algorithmic}

\usepackage{hyperref}
\usepackage{url}

\usepackage{amsmath}
\usepackage{amssymb}
\usepackage{mathtools}
\usepackage{amsthm}
\usepackage{bbm}
\usepackage{nicefrac}

\usepackage[capitalize,noabbrev]{cleveref}

\theoremstyle{plain}
\newtheorem{theorem}{Theorem}[section]

\newtheorem{lemma}[theorem]{Lemma}

\theoremstyle{definition}
\newtheorem{definition}[theorem]{Definition}
\newtheorem{assumption}[theorem]{Assumption}

\theoremstyle{remark}

\Crefname{assumption}{Assumption}{Assumptions}
\crefname{assumption}{Assumption}{Assumptions}

\allowdisplaybreaks

\makeatletter
\renewcommand{\paragraph}{%
  \@startsection{paragraph}{4}%
  {\z@}{0ex \@plus 0ex \@minus 0ex}{-1em}%
  {\normalfont\normalsize\bfseries}%
}
\makeatother

\newlist{lemenum}{enumerate}{1} %
\setlist[lemenum]{label=(\roman*),ref=\thelemma\,(\roman*),topsep=0pt}
\Crefname{lemenumi}{Lemma}{Lemmas}

\newlist{corenum}{enumerate}{1} %
\setlist[corenum]{label=(\roman*),ref=\thecorollary\,(\roman*),topsep=0pt}
\Crefname{corenumi}{Corollary}{Corollaries}

\newlist{thmenum}{enumerate}{1} %
\setlist[thmenum]{label=(\roman*),ref=\thetheorem\,(\roman*),topsep=0pt}
\Crefname{thmenumi}{Theorem}{Theorems}

\newlist{propenum}{enumerate}{1} %
\setlist[propenum]{label=(\roman*),ref=\thedefinition\,(\roman*),topsep=0pt}
\Crefname{propenumi}{Property}{Properties}

\newlist{assenum}{enumerate}{1} %
\setlist[assenum]{label=(\roman*),ref=\theassumption\,(\roman*),topsep=0pt}
\Crefname{assenumi}{Assumption}{Assumptions}

\usepackage{svg}
\usepackage{graphicx}
\usepackage{rotating}
\usepackage{tikz}
\usepackage{pgfplots}
\pgfplotsset{compat=newest}
\usetikzlibrary{shapes.geometric}

\definecolor{chaptercolor}{HTML}{1A254B}
\definecolor{darkblue}{HTML}{1A254B}
\definecolor{linkcolor}{HTML}{2B50AA}
\definecolor{citecolor}{HTML}{2B50AA}
\definecolor{linkcolor}{HTML}{2B50AA}
\definecolor{lightlinkcolor}{HTML}{9A8F97}
\definecolor{darklinkcolor}{HTML}{1A254B}
\definecolor{pink}{HTML}{E05F60}
\definecolor{lightblue}{HTML}{A7BED3}
\definecolor{red}{HTML}{F2545B}
\definecolor{blue}{HTML}{2b50aa}
\definecolor{m1}{HTML}{1f77b4}
\definecolor{m2}{HTML}{ff7f0e}
\definecolor{m3}{HTML}{2ca02c}
\definecolor{m4}{HTML}{d62728}
\definecolor{m5}{HTML}{9467bd}
\definecolor{m6}{HTML}{8c564b}
\definecolor{m7}{HTML}{e377c2}
\definecolor{m8}{HTML}{7f7f7f}
\definecolor{m9}{HTML}{bcbd22}
\definecolor{m10}{HTML}{17becf}

\usepackage{import}
\usepackage{xifthen}
\usepackage{pdfpages}
\usepackage{transparent}

\NewDocumentCommand{\incfig}{mo}{
  \begin{center}
    \IfValueT{#2}{\def\svgwidth{#2}}{\def\svgwidth{\columnwidth}}
    \import{./figures/}{#1.pdf_tex}
  \end{center}
}

\usepackage{pgf}
\usepackage{adjustbox}

\NewDocumentCommand{\incplt}{O{\columnwidth}m}{%
  \begin{center}
    \adjustbox{width=#1}{\import{./plots/output/}{#2.pgf}}
  \end{center}
}

\makeatletter
\newenvironment{chapquote}[2][2em]
  {%
   \def\chapquote@author{#2}%
   \itshape}
  {\par\normalfont\hfill--\ \chapquote@author\par}
\makeatother

\usepackage{aligned-overset}
\newcommand{\rmnum}[1]{\romannumeral #1}
\newcommand{\egeq}[1]{\overset{(\rmnum{#1})}&{\geq}}
\newcommand{\eleq}[1]{\overset{(\rmnum{#1})}&{\leq}}
\newcommand{\eeq}[1]{\overset{(\rmnum{#1})}&{=}}
\newcommand{\e}[1]{\ensuremath{(\rmnum{#1})}}
\newcommand{\ebf}[1]{\ensuremath{\textbf{(\rmnum{#1})}}}

\newcommand{\itl}{\textcolor{red}{\textbf{\texttt{ITL}}}\xspace}
\newcommand{\gitl}{\textcolor{red}{\textbf{\texttt{G-ITL}}}\xspace}
\newcommand{\litl}{\textcolor{red}{\textbf{\texttt{L-ITL}}}\xspace}
\newcommand{\ctl}{\textbf{\texttt{CTL}}\xspace}
\newcommand{\gctl}{\textbf{\texttt{G-CTL}}\xspace}
\newcommand{\lctl}{\textbf{\texttt{L-CTL}}\xspace}

\DeclareFontFamily{U}{mathb}{\hyphenchar\font45}
\DeclareFontShape{U}{mathb}{m}{n}{
      <5> <6> <7> <8> <9> <10> gen * mathb
      <10.95> mathb10 <12> <14.4> <17.28> <20.74> <24.88> mathb12
      }{}
\DeclareSymbolFont{mathb}{U}{mathb}{m}{n}
\DeclareFontSubstitution{U}{mathb}{m}{n}
\DeclareMathSymbol{\Asterisk}      {2}{mathb}{"06}
\newcommand{\mainsym}{\textcolor{blue}{\parentheses*{\hspace{-0.1em}\Asterisk\hspace{-0.1em}}}}

\newcommand{\psym}{\ensuremath{\textcolor{blue}{(\dagger)}}\xspace}

\newcommand*{\abs}[1]{| #1 |}

\NewDocumentCommand{\norm}{sm}{\IfBooleanTF{#1}{\|#2\|}{\left\| #2 \right\|}}

\newcommand*{\const}{\mathrm{const}}

\DeclareMathOperator*{\defeq}{\smash{\overset{\mathrm{def}}{=}}}
\DeclareMathOperator*{\eqdef}{\smash{\overset{\mathrm{def}}{=}}}

\DeclareMathOperator*{\argmax}{arg\,max}
\DeclareMathOperator*{\argmin}{arg\,min}

\DeclarePairedDelimiter\parentheses{(}{)}
\DeclarePairedDelimiter\brackets{[}{]}
\DeclarePairedDelimiter\braces{\{}{\}}

\newcommand{\R}{\mathbb{R}}

\renewcommand{\vec}[1]{\boldsymbol{#1}}
\newcommand{\mat}[1]{\boldsymbol{#1}}

\newcommand{\spa}[1]{\mathcal{#1}}
\newcommand{\opt}[1]{#1^\star}

\DeclareMathOperator*{\iid}{\smash{\overset{\mathrm{iid}}{\sim}}}
\DeclareMathOperator*{\uar}{\smash{\overset{\mathrm{u.a.r.}}{\sim}}}

\NewDocumentCommand{\Ind}{m}{\mathbbm{1}\{{#1}\}}
\NewDocumentCommand{\fnPr}{}{\mathbb{P}}
\RenewDocumentCommand{\Pr}{om}{\fnPr\IfValueT{#1}{_{#1}}\parentheses*{#2}}
\RenewDocumentCommand{\H}{mo}{\mathrm{H}\IfValueTF{#2}{\!\left[#1\ \middle|\ #2\right]}{\brackets*{#1}}}
\NewDocumentCommand{\Hsm}{mo}{\mathrm{H}\IfValueTF{#2}{[#1 \mid #2]}{\brackets{#1}}}
\NewDocumentCommand{\I}{mmo}{\mathrm{I}\IfValueTF{#3}{\!\left(#1;#2\ \middle|\ #3\right)}{\parentheses*{#1; #2}}}
\NewDocumentCommand{\Ism}{mmo}{\mathrm{I}\IfValueTF{#3}{(#1;#2 \mid #3)}{\parentheses{#1; #2}}}

\NewDocumentCommand{\E}{somo}{\ensuremath{\mathbb{E}\IfValueT{#2}{_{#2}}{} \IfBooleanTF{#1}{#3}{\IfValueTF{#4}{\!\left[#3\ \middle|\ #4\right]}{\brackets*{#3}}}}}
\NewDocumentCommand{\Esm}{somo}{\ensuremath{\mathbb{E}\IfValueT{#2}{_{#2}}{} \IfBooleanTF{#1}{#3}{\IfValueTF{#4}{\!\left[#3\ \middle|\ #4\right]}{\brackets{#3}}}}}
\NewDocumentCommand{\Var}{somo}{\mathrm{Var}\IfValueT{#2}{_{#2}}{} \IfBooleanTF{#1}{#3}{\IfValueTF{#4}{\!\left[#3\ \middle|\ #4\right]}{\brackets*{#3}}}}
\NewDocumentCommand{\Varsm}{somo}{\mathrm{Var}\IfValueT{#2}{_{#2}}{} \IfBooleanTF{#1}{#3}{\IfValueTF{#4}{\left[#3\ \middle|\ #4\right]}{\brackets{#3}}}}
\NewDocumentCommand{\Cov}{som}{\mathrm{Cov}\IfValueT{#2}{_{#2}}{} \IfBooleanTF{#1}{#3}{\brackets*{#3}}}
\NewDocumentCommand{\Cor}{som}{\mathrm{Cor}\IfValueT{#2}{_{#2}}{} \IfBooleanTF{#1}{#3}{\brackets*{#3}}}

\NewDocumentCommand{\grad}{e_}{\boldsymbol{\nabla}\IfValueT{#1}{_{\!\!#1}\,}}

\NewDocumentCommand{\BigO}{m}{O\parentheses*{#1}}
\NewDocumentCommand{\BigOTil}{m}{\widetilde{O}\parentheses*{#1}}

\NewDocumentCommand{\transpose}{m}{#1^\top}
\NewDocumentCommand{\inv}{m}{#1^{-1}}
\RenewDocumentCommand{\det}{m}{\left| #1 \right|}

\NewDocumentCommand{\tr}{m}{\mathrm{tr}\;#1}
\NewDocumentCommand{\diag}{som}{\mathrm{diag}\IfValueT{#2}{_{#2}}{} \IfBooleanTF{#1}{\braces{#3}}{\braces*{#3}}}
\NewDocumentCommand{\msqrt}{m}{#1^{\nicefrac{1}{2}}}

\NewDocumentCommand{\N}{somm}{\mathcal{N}\IfBooleanTF{#1}{\left(}{(}\IfValueT{#2}{#2;}{} #3, #4\IfBooleanTF{#1}{\right)}{)}}
\NewDocumentCommand{\GP}{omm}{\mathcal{GP}(\IfValueT{#1}{#1;}{} #2, #3)}

\newcommand{\vzero}{\vec{0}}

\newcommand{\vf}{\vec{f}}

\newcommand{\vfsub}[1]{\vec{f}_{\!\!#1}}

\newcommand{\vk}{\vec{k}}

\newcommand{\vx}{\vec{x}}

\newcommand{\vxp}{\vec{x'}}

\newcommand{\vy}{\vec{y}}
\newcommand{\vysub}[1]{\vec{y}_{\!#1}}

\newcommand{\vz}{\vec{z}}
\newcommand{\vbeta}{\boldsymbol{\beta}}

\newcommand{\vmu}{\boldsymbol{\mu}}
\newcommand{\vmusub}[1]{\boldsymbol{\mu}_{\!#1}}

\newcommand{\vphi}{\boldsymbol{\phi}}

\newcommand{\vtheta}{\boldsymbol{\theta}}
\newcommand{\vthetap}{\boldsymbol{\theta'}}
\newcommand{\vthetahat}{\boldsymbol{\widehat{\theta}}}

\newcommand{\mzero}{\mat{0}}
\newcommand{\mA}{\mat{A}}

\newcommand{\mD}{\mat{D}}

\newcommand{\mI}{\mat{I}}

\newcommand{\mK}{\mat{K}}
\newcommand{\mKsub}[1]{\mat{K}_{\!#1}}

\newcommand{\mPsub}[1]{\mat{P}_{\!#1}}
\newcommand{\mPhi}{\mat{\Phi}}

\newcommand{\mQ}{\mat{Q}}

\newcommand{\mSigma}{\mat{\Sigma}}

\newcommand{\spA}{\spa{A}}
\newcommand{\spB}{\spa{B}}

\newcommand{\spD}{\spa{D}}

\newcommand{\spH}{\spa{H}}

\newcommand{\spS}{\spa{S}}

\newcommand{\spX}{\spa{X}}
\newcommand{\spY}{\spa{Y}}

\newcommand{\spPA}{\ensuremath{\spa{P}_{\!\!\spA}}}
\newcommand{\spPS}{\ensuremath{\spa{P}_{\!\spS}}}

\title{Active Few-Shot Fine-Tuning}

\author{Jonas Hübotter\thanks{Correspondence to \texttt{jonas.huebotter@inf.ethz.ch}}\ , Bhavya Sukhija, Lenart Treven, Yarden As, Andreas Krause \\
ETH Zurich, Switzerland
}

\iclrfinalcopy %
\begin{document}

\maketitle

\begin{abstract}

   We study the question: \emph{How can we select the right data for fine-tuning to a specific task?}
   We call this data selection problem \emph{active fine-tuning} and show that it is an instance of transductive active learning, a novel generalization of classical active learning.
   We propose \itl, short for \emph{\underline{i}nformation-based \underline{t}ransductive \underline{l}earning}, an approach which samples adaptively to maximize information gained about the specified task.
   We are the first to show, under general regularity assumptions, that such decision rules converge uniformly to the smallest possible uncertainty obtainable from the accessible data.
   We apply \itl to the few-shot fine-tuning of large neural networks and show that fine-tuning with \itl learns the task with significantly \emph{fewer} examples than the state-of-the-art.
\end{abstract}

\section{Introduction}\label{sec:introduction}

Despite the remarkable successes of large neural networks (NNs) across various fields, such as image classification and natural language processing, their performance can deteriorate when faced with slight variations between the source and target domains~\citep{recht2019imagenet,hendrycks2019benchmarking,koh2021wilds,lee2022surgical}.
Additionally, training large NNs requires large amounts of (labeled) data, which is often expensive or even impossible to obtain, and furthermore, training on such large datasets requires prohibitive computational resources.\looseness -1

Fine-tuning a large pre-trained model on a (small) dataset from the target domain is a cost- and computation-effective approach to address the distribution shift between source and target domains.
While previous work has studied the effectiveness of various training procedures for fine-tuning~\citep{howard2018universal,kornblith2019better,shen2021partial,lee2022surgical,silva2023towards}, the problem of obtaining a good dataset for fine-tuning has received less attention.
Selecting such a \emph{small} dataset is challenging, as it requires selecting the most relevant and diverse data from a large dataset based only on a few reference examples from the target domain.\looseness -1

\begin{wrapfigure}{r}{0.4\textwidth}
  \vspace{-0.5cm}
  \centering
  \includesvg[width=0.4\columnwidth]{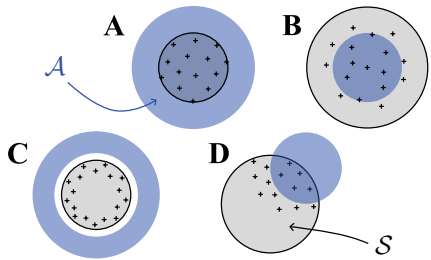}
  \vspace{-0.5cm}
  \caption{Instances of transductive active learning where the target~space~$\spA$ is shown in blue and the sample~space~$\spS$ is shown in gray. The points denote plausible observations within~$\spS$ to ``learn''~$\spA$. In \textbf{(A)}, the target space contains ``everything'' within~$\spS$ as well as points \emph{outside}~$\spS$. In \textbf{(B, C, D)}, one makes observations \emph{directed} towards learning about a particular target. Prior work on active learning has focused on the instances $\spA = \spS$ and $\spA \subset \spS$.\looseness=-1}
  \label{fig:transductive_active_learning}
  \vspace{-2.2cm}
\end{wrapfigure}

In this work, we propose a generalization of classical active learning, \emph{transductive active learning}.
We show that the fine-tuning of large neural networks can be seen as transductive active learning, and propose \itl which significantly improves upon the state-of-the-art --- \emph{enabling efficient few-shot fine-tuning of large neural networks}.\looseness -1

\paragraph{Transductive active learning}

We consider the problem of \emph{transductive active learning}, where provided an unknown and sufficiently regular function~$f$ over a domain~$\spX$ and given two arbitrary subsets of the domain; a \emph{target space} $\spA \subseteq \spX$, and a \emph{sample space} $\spS \subseteq \spX$, we study the question: \begin{center}
    \emph{How can we learn $f$ within $\spA$ \\ by actively sampling observations within $\spS$?}
\end{center}

Prior work on active learning commonly aims to learn $f$ globally, i.e., across the entire domain~$\spX$.
However, in many real-world problems, \ebf{1} the domain $\spX$ is so large that learning~$f$ globally is hopeless or \ebf{2} agents have limited information and cannot access the entire domain $\spX$ (e.g., due to restricted access or to act safely). %
Thus, global learning is often not desirable or even possible.
Instead, intelligent systems are typically required to act in a more \emph{directed} manner and \emph{extrapolate} beyond their limited information.
This work studies the above two aspects of active learning, which have remained largely unaddressed by prior work.
We provide a comprehensive overview of related work in \cref{sec:related_work}.\looseness-1

The fine-tuning of large neural networks is an instance of transductive active learning, where the target space~$\spA$ represents the test set over which we want to minimize risk, and the sample space~$\spS$ represents the dataset from which we can retrieve data points to fine-tune our model to~$\spA$.
\Cref{fig:transductive_active_learning} visualizes some instances of transductive active learning.
Whereas most prior work has focused on the instance ${\spX = \spA = \spS}$, \cite{mackay1992information} was the first to consider specific target spaces~$\spA$ as in (B), and transductive active learning generalizes to other instances such as (A), (C), and (D).\looseness=-1

\paragraph{Contributions}

\begin{itemize}%
    \item We are the first to give rates for the uniform convergence of uncertainty over the target space~$\spA$ to the smallest attainable value, given samples from the sample space~$\spS$~(\cref{thm:variance_convergence}), which implies a new generalization bound for functions in reproducing kernel Hilbert spaces~(\cref{thm:width_convergence}).

    \item We apply the transductive active learning framework to batch-wise \emph{active few-shot fine-tuning of large neural networks} and show empirically that \itl substantially outperforms the state-of-the-art~(\cref{sec:nns}).
\end{itemize}

\section{Preliminaries}\label{sec:problem setting}

We assume that the target space $\spA$ and sample space $\spS$ are finite.\footnote{Infinite domains can be addressed via discretization arguments from the Bayesian optimization literature.}
We model~$f$ as a stochastic process and denote the marginal random variables $f(\vx)$ by $f_{\vx}$, and joint random vectors $\{f(\vx)\}_{\vx \in X}$ for some ${X \subseteq \spX}, {\abs{X} < \infty}$ by $\vfsub{X}$.
Let $\vysub{X}$ denote the noisy observations of $\vfsub{X}$, ${\{y_{\vx} = f_{\vx} + \varepsilon_{\vx}\}_{\vx \in X}}$, where $\varepsilon_{\vx}$ is independent noise.\footnote{$X$ may be a multiset, in which case repeated $\vx$ correspond to repeated independent observations of~$y_{\vx}$.}
We study the ``adaptive'' setting, where in round $n$ the agent selects a point $\vx_n \in \spS$ and observes ${y_n = y_{\vx_n}}$.
The agent's choice of $\vx_n$ may depend on the prior observations $\spD_{n-1} \defeq \{(\vx_i, y_i)\}_{i<n}$.\looseness=-1

\paragraph{Background on information theory}

We briefly recap several important concepts from information theory, of which we provide formal definitions in \cref{sec:definitions}.
The (differential) entropy~$\H{\vf}$ is a measure of uncertainty about~$\vf$ and the conditional entropy $\H{\vf}[\vy]$ is the (expected) posterior uncertainty about~$\vf$ after observing~$\vy$.
The information gain~${\I{\vf}{\vy} = \H{\vf} - \H{\vf}[\vy]}$ measures the (expected) reduction in uncertainty about $\vf$ due to $\vy$.
The maximum information gain from $n$~noisy observations within $\spS$ is \begin{align*}
    \gamma_n \defeq \max_{\substack{X \subseteq \spS \\ \abs{X} \leq n}} \I{\vfsub{X}}{\vysub{X}}.
\end{align*}
This has been used previously \citep[e.g., by][]{srinivas2009gaussian,chowdhury2017kernelized,vakili2021information} as a measure of the ``information capacity'' of $f$.\looseness=-1

\section{Main Results on Transductive Active Learning}\label{sec:algorithm}

We propose \itl, which greedily maximizes the information gain at each round~$n$ between the prediction targets~$\vfsub{\spA}$ and the observation~$y_{\vx}$ conditioned on the prior observations~$\spD_{n-1}$.
Formally,
\begin{align*}
    \vx_{n} = \argmax_{\vx \in \spS} \Ism{\vfsub{\spA}}{y_{\vx}}[\spD_{n-1}]. \tag{\itl}
\end{align*}
This simple decision rule generalizes several widely used algorithms which we discuss in more detail in \cref{sec:related_work}.

We also consider an additional correlation-based decision rule, which will later uncover connections to existing approaches: \begin{align*}
    \vx_{n} &= \argmax_{\vx \in \spS} \sum_{\vxp \in \spA} \Cor{f_{\vx},f_{\vxp} \mid \spD_{n-1}}. \tag{\ctl}
\end{align*}
Unlike \ctl, \itl takes into account the mutual dependence between points in~$\spA$.\looseness=-1

In the following \cref{sec:itl:gp_setting,sec:algorithm:agnostic}, we discuss general settings in which the convergence properties of \itl can be analyzed theoretically.\looseness=-1

\subsection{Gaussian Process Setting}\label{sec:itl:gp_setting}

When $f \sim \GP{\mu}{k}$ is a Gaussian process (GP, \cite{williams2006gaussian}) with known mean function $\mu$ and kernel $k$, and the noise $\varepsilon_{\vx}$ is mutually independent and zero-mean Gaussian with known variance $\rho^2(\vx) > 0$, the \itl objective has a closed form expression: \begin{align}
    \Ism{\vfsub{\spA}}{y_{\vx}}[\spD_{n-1}] = \frac{1}{2} \log\parentheses*{\frac{\Var{y_{\vx} \mid \spD_{n-1}}}{\Var{y_{\vx} \mid \vfsub{\spA}, \spD_{n-1}}}}. \label{eq:objective_gp_setting}
\end{align}
Further, the information capacity $\gamma_n$ is sublinear in $n$ for a rich class of GPs, with rates summarized in \cref{table:gamma_rates} of the appendix \citep{srinivas2009gaussian,vakili2021information}.

\paragraph{Convergence to irreducible uncertainty}

So far, our discussion was centered around the role of the target space $\spA$ in facilitating \emph{directed} learning.
An orthogonal contribution of this work is to study \emph{extrapolation} from the sample space $\spS$ to points ${\vx \in \spA \setminus \spS}$.
To this end, we derive bounds on the marginal posterior variance $\sigma_n^2(\vx) \defeq \Var{f(\vx) \mid \spD_n}$ for points in $\spA$.
For ease of presentation, we assume in the following that $\spS \subseteq \spA$, however, the results are straightforward to generalize to other instances \citep{hubotter2024information}.
For \itl, these bounds imply uniform convergence of the variance for a rich class of GPs.
To the best of our knowledge, we are the first to derive such bounds, and they might be of independent interest for active learning.\looseness=-1

\looseness=-1 We define the \emph{irreducible uncertainty} as the variance of $f(\vx)$ provided complete knowledge of $f$ in $\spS$: \begin{equation*}
    \eta^2_{\spS}(\vx) \defeq \Varsm{f_{\vx} \mid \vfsub{\spS}}.
\end{equation*}
As the name suggests, $\eta^2_{\spS}(\vx)$ represents the smallest uncertainty one can hope to achieve from observing only within $\spS$.
For all $\vx \in \spS$, it is easy to see that $\eta^2_{\spS}(\vx) = 0$.
However, the irreducible uncertainty of $\vx \not\in \spS$ may be (and typically is!) strictly positive.\looseness=-1

\begin{theorem}[Generalization bound on marginal variance for \itl]\label{thm:variance_convergence}
   \looseness=-1 Assume that ${f \sim \GP{\mu}{k}}$ with known mean function $\mu$ and kernel $k$, the noise $\varepsilon_{\vx}$ is mutually independent and zero-mean Gaussian with known variance $\rho^2(\vx) > 0$, and $\gamma_n$ is sublinear in $n$.
   Then, for any $n \geq 1, \epsilon > 0$, and $\vx \in \spA$, there exists a constant $C$ such that \begin{align}
        \sigma_n^2(\vx) \leq \underbrace{\eta^2_{\spS}(\vx)}_{\text{irreducible}} + \underbrace{C \gamma_n / \sqrt{n}}_{\text{reducible}}. \label{eq:variance_convergence:extrapolation}
    \end{align}
\end{theorem}
Intuitively, \cref{eq:variance_convergence:extrapolation} of \cref{thm:variance_convergence} can be understood as bounding an epistemic ``generalization gap'' \citep{wainwright2019high} of the learner.
The reducible uncertainty converges to zero, e.g., for Gaussian and smooth Matérn kernels, at all prediction targets $\vx \in \spA$.
We provide a formal proof of \cref{thm:variance_convergence} in \cref{sec:proofs:variance_convergence}.\looseness=-1

\subsection{Agnostic Setting}\label{sec:algorithm:agnostic}

The result from the GP setting translates also to the agnostic setting, where the ``ground truth'' $\opt{f}$ may be any sufficiently regular fixed function on $\spX$.\footnote{Here $\opt{f}(\vx)$ denotes the mean observation ${y_{\vx} = \opt{f}(\vx) + \epsilon_{\vx}}$}
In this case, we use the model~$f$ from \cref{sec:itl:gp_setting} as a (misspecified) model of~$\opt{f}$, with some kernel $k$ and zero mean function ${\mu(\cdot) = 0}$.
We denote by $\mu_n(\vx) \defeq \E{f(\vx)}[\spD_n]$ the posterior mean of~$f$.
W.l.o.g. we assume in the following result that the prior variance is bounded, i.e., $\Var{f(\vx)} \leq 1$.\footnote{The results can be generalized to kernels with $\Var{f(\vx)} \leq c$ \citep{chowdhury2017kernelized}.}\looseness=-1

\begin{theorem}[Bound on generalization error for \itl, following \cite{abbasi2013online,chowdhury2017kernelized}]\label{thm:width_convergence}
    Pick any $\delta \in (0,1)$.
    Assume that $\opt{f}$ lies in the reproducing kernel Hilbert space $\spH_k(\spX)$ of the kernel $k$ with norm $\smash{\norm{\opt{f}}_k < \infty}$, the noise $\varepsilon_n$ is conditionally $\rho$-sub-Gaussian, and $\gamma_n$ is sublinear in $n$.
    Let $\smash{\beta_{n}(\delta) = \norm{\opt{f}}_k + \rho \sqrt{2(\gamma_{n} + 1 + \log(1 / \delta))}}$.
    Then, for any $n \geq 1$ and $\vx \in \spA$, jointly with probability at least $1-\delta$, \begin{align*}
        |\opt{f}(\vx) - \mu_{n}(\vx)| \leq \beta_{n}(\delta) \Big[ \underbrace{\eta_{\spS}(\vx)}_{\text{irreducible}} + \underbrace{\nu(n)}_{\text{reducible}} \Big]
    \end{align*} where $\nu^2(n)$ denotes the reducible part of \cref{eq:variance_convergence:extrapolation}.
\end{theorem}

We provide a formal proof of \cref{thm:width_convergence} in \cref{sec:proofs:width_convergence}.
\Cref{thm:width_convergence} generalizes confidence bounds of prior works to the extrapolation setting, where some prediction targets~$\vx \in \spA$ lie outside the sample space~$\spS$.
\Cref{thm:width_convergence} can therefore be interpreted as a \emph{generalization bound} which holds uniformly for all functions $\opt{f}$ with $\norm{\opt{f}}_k < \infty$.
For prediction targets~${\vx \in \spA \cap \spS}$, the irreducible uncertainty vanishes, and we recover previous results from the setting $\spS = \spA$.\looseness=-1

\Cref{thm:variance_convergence,thm:width_convergence} show that \itl efficiently learns $f$ at the prediction targets $\spA$ for large classes of ``sufficiently regular'' functions $f$.
In the following, we validate these results experimentally by showing that \itl exhibits strong empirical performance for fine-tuning large neural networks.\looseness=-1

\section{Few-Shot Fine-Tuning of Neural Networks}\label{sec:nns}

As discussed in \cref{sec:introduction}, the fine-tuning of large neural networks to downstream tasks can be framed as a transductive active learning problem.
In our experiments, we consider a supervised learning problem, where the function $f$ maps inputs ${\vx \in \spX}$ to outputs ${y \in \spY}$.\footnote{We remark that also the semi-supervised or unsupervised fine-tuning of models can be framed as a transductive active learning problem.}
We have access to noisy samples from a training set~$\spS$ on~$\spX$, and we would like to learn~$f$ such that our estimate minimizes a given risk measure, such as classification error, with respect to a test distribution~$\spPA$ on~$\spX$.
Ideally, $\spPS$ and~$\spPA$ are identical, however, in many real-world cases $\spPS$ differs substantially from~$\spPA$.
This is for instance the case when samples from~$\spPA$ are expensive to obtain, such as in medical imaging.
The goal is to actively and efficiently sample from~$\spS$ to minimize risk with respect to~$\spPA$.\footnote{The setting with target distributions $\spPA$ can be reduced to considering target sets $\spA$, cf.~\cref{sec:generalizations:roi}.}
We show in this section that \itl can learn $f$ from only \emph{few examples} from $\spS$.\looseness -1%

\paragraph{How can we leverage the latent structure learned by the pre-trained model?}

As common in related work, we approximate the (pre-trained) neural network (NN) $f(\cdot; \vtheta)$ as a linear function in a latent embedding space, $\smash{f(\vx; \vtheta) \approx \transpose{\vbeta} \vphi_{\vtheta}(\vx)}$, with weights $\vbeta \in \R^p$ and embeddings $\vphi_{\vtheta} : \spX \to \R^p$.
Common choices of embeddings include last-layer embeddings \citep{devlin2018bert,holzmuller2023framework}, neural tangent embeddings arising from neural tangent kernels \citep{jacot2018neural} which are motivated by their relationship to the training and fine-tuning of ultra-wide NNs \citep{arora2019exact,lee2019wide,khan2019approximate,he2020bayesian,malladi2023kernel}, and loss gradient embeddings \citep{ash2020deep}.
We provide a comprehensive overview of embeddings in \cref{sec:nns_appendx:embeddings}.
Now, supposing the prior $\vbeta \sim \N{\vzero}{\mSigma}$, often with ${\mSigma = \mI}$~\citep{khan2019approximate,he2020bayesian,antoran2022adapting,wei2022more}, this approximation of $f$ is a Gaussian process with kernel $k(\vx,\vxp) = \transpose{\vphi_{\vtheta}(\vx)} \mSigma \vphi_{\vtheta}(\vxp)$ which quantifies the similarity between points in terms of their alignment in the learned latent space.
Note that the correlation $\smash{k(\vx, \vxp) / \sqrt{k(\vx,\vx) k(\vxp,\vxp)}}$ between two points $\vx, \vxp$ is equal to the cosine similarity of their embeddings.

In this context, \cref{thm:variance_convergence} bounds the epistemic posterior uncertainty about a prediction using the approximation $\smash{\transpose{\vbeta} \vphi_{\vtheta}(\vx)}$, given that the model is trained using data selected by \itl.
\Cref{thm:width_convergence} bounds the generalization error when using the posterior mean of $\vbeta$ for prediction.
This extends recent work which has studied estimators of this generalization error \citep{wei2022more}.\looseness -1

\paragraph{Batch selection: Diversity via conditional embeddings}

Efficient labeling and training necessitates a batch-wise selection of inputs.
The selection of a batch of size $b > 1$ can be seen as an individual \emph{non-adaptive} active learning problem, and significant recent work has shown that batch diversity is crucial in this setting \citep{ash2020deep,zanette2021design,holzmuller2023framework,pacchiano2024experiment}.
A batch-wise selection strategy is formalized by the following NP-hard non-adaptive transductive active learning problem \citep{krause2014submodular,golovin2011adaptive}: \begin{align}
  B_{n} = \argmax_{\substack{B \subseteq \spS, \abs{B} = b}} \Ism{\vfsub{\spA}}{\vysub{B}}[\spD_{n-1}]. \label{eq:batch_selection}
\end{align}
\itl solves this optimization problem greedily: \begin{align}
  \vx_{n,i} = \argmax_{\vx \in \spS} \Ism{\vfsub{\spA}}{y_{\vx}}[\spD_{n-1}, \vysub{\vx_{n,1:i-1}}].\label{eq:greedy_batch_selection}
\end{align}
We show in \cref{sec:nonmyopic} that the approximation error of $B'_n = \vx_{n,1:b}$ can be bounded in terms of the submodularity ratio of \itl \citep{das2018approximate}.
For example, if $\spS \subseteq \spA$, the problem from \cref{eq:batch_selection} is submodular and the $B'_n$ is a $(1-1/e)$-approximation of $B_n$ \citep{nemhauser1978analysis,krause2014submodular}.
The batch $B_n$, and therefore also $B'_n$, is diverse and informative by design.
We discuss an efficient iterative implementation of \cref{eq:greedy_batch_selection} in \cref{sec:nns_appendx:batch_selection}.
Prior work has shown that the greedy solution $B'_n$ is also competitive with a fully sequential ``batchless'' decision rule \citep{chen2013near,esfandiari2021adaptivity}.\footnote{They prove this for the case where \cref{eq:batch_selection} is submodular, but their results readily generalize to ``approximate'' submodularity.}\looseness=-1

\subsection{Experiments}

Our empirical evaluation is motivated by the following practical example:
We deploy a pre-trained image classifier to user's phones who use it within their local environment.
We would like to locally fine-tune a user's model to their environment.
Since the users' images~$\spA$ are unlabeled, this requires selecting a small number of relevant and diverse images from the set of labeled images~$\spS$.
As such, we will focus here on the setting where the points in our test set do not lie in our training set~(i.e.,~${\spA \cap \spS = \emptyset}$), while \cite{hubotter2024information} discusses also alternative settings such as active domain adaptation.\looseness=-1

\paragraph{Testbeds \& architectures}

We use the MNIST \citep{lecun1998mnist} and CIFAR-100 \citep{krizhevsky2009learning} datasets as testbeds.
In both cases, we take $\spS$ to be the training set, and we consider the task of learning the digits $3$, $6$, and $9$ (MNIST) or the first $10$ categories of CIFAR-100.\footnote{That is, we restrict $\spPA$ to the support of points with labels $\{3,6,9\}$ (MNIST) or labels $\{0, \dots, 9\}$ (CIFAR-100) and train a neural network using few examples drawn from the training set $\spS$.}
For MNIST, we train a simple convolutional neural network with ReLU activations, three convolutional layers with max-pooling, and two fully-connected layers.
For CIFAR-100, we fine-tune an EfficientNet-B0 \citep{tan2019efficientnet} pre-trained on ImageNet \citep{deng2009imagenet}, augmented by a final fully-connected layer.
We train the NNs using the cross-entropy loss and the ADAM optimizer \citep{kingma2014adam}.\looseness=-1

\paragraph{Results}

In \cref{fig:nns}, We compare against \textbf{(i)} active learning methods which largely aim for sample diversity but which are not directed towards the target distribution $\spPA$ \citep[e.g., BADGE;][]{ash2020deep}, and \textbf{(ii)} search methods that aim to retrieve the most relevant samples from $\spS$ with respect to the targets $\spPA$ \citep[e.g., maximizing cosine similarity to target embeddings as is common in vector databases;][]{settles2008analysis,johnson2019billion}. \textsc{InformationDensity} \citep[ID,][]{settles2008analysis} is a heuristic approach aiming to combine \textbf{(i)} diversity and \textbf{(ii)} relevance.
We observe that \itl and \ctl consistently and significantly outperform random sampling from $\spS$ as well as all baselines.
We see that relevance-based methods such as \textsc{CosineSimilarity} have an initial advantage over \textsc{Random} but for batch sizes larger than $1$ they quickly fall behind due to diminishing informativeness of the selected data.
In contrast, diversity-based methods such as \textsc{BADGE} are more competitive with \textsc{Random} but do not explicitly aim to retrieve relevant samples.

\begin{figure*}[t]
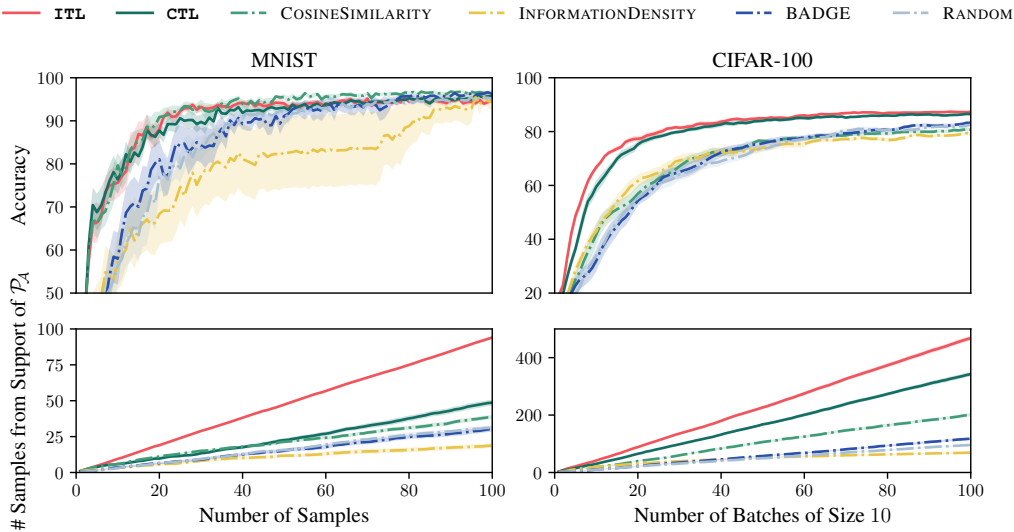

    \incplt[\textwidth]{nns}
    \caption{Few-shot training of NNs on MNIST (left) and CIFAR-100 (right). \textsc{Random} selects each observation uniformly at random from $\spPS$. The batch size is $1$ for MNIST and $10$ for CIFAR-100. Uncertainty bands correspond to one standard error over $10$ random seeds. We see that \itl significantly outperforms the state-of-the-art, and in particular, retrieves substantially more samples from the support of $\spPA$ than competing methods. This trend becomes even more pronounced in more difficult large-scale learning tasks (cf.~\cref{fig:nns_imbalanced_train} in \cref{sec:nns_appendix}). See \cref{sec:nns_appendix} for details and additional experiments.}
    \label{fig:nns}
\end{figure*}

\begin{figure*}[]
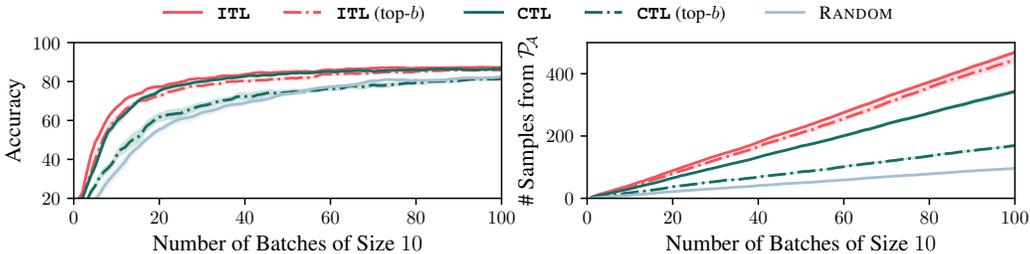

  \incplt[\textwidth]{nns_batch_selection}
  \caption{Batch selection via conditional embeddings improves substantially over selecting the top-$b$ candidates proposed by the decision rule. This is the CIFAR-100 experiment (where $b=10$).}
  \label{fig:nns_batch_selection}
\end{figure*}

In \cref{fig:nns_batch_selection}, we compare \emph{batch selection via conditional embeddings} (\textsc{BaCE}) to selecting the top-$b$ points according to the decision rule (which does \emph{not} yield diverse batches).
We observe a significant improvement in accuracy and data retrieval when using \textsc{BaCE}.
We expect the gap between both approaches to widen further with larger batch sizes.

\paragraph{Synthesizing sample relevance and diversity}

Our proposed methods synthesize approaches to coverage (promoting \emph{diverse} samples) and search (aiming for \emph{relevant} samples with respect to a given query~$\spA$) which leads to the significant improvement upon the state-of-the-art in \cref{fig:nns}.
Notably, for a batch size and query size of $1$ and if correlations are non-negative, \itl, \ctl, and the canonical cosine similarity are equivalent.
\ctl can be seen as a direct generalization of cosine similarity-based retrieval to batch and query sizes larger than one.
In contrast to \ctl, \itl may also sample points which exhibit a strong negative correlation (which is also informative).\looseness=-1

We observe empirically that \itl obtains samples from $\spPA$ at more than twice the rate of \textsc{CosineSimilarity}, which translates to a significant improvement in accuracy in more difficult learning tasks, while requiring fewer (labeled) samples from $\spS$.
This phenomenon manifests for both MNIST and CIFAR-100, as well as imbalanced datasets $\spS$ or imbalanced reference samples from $\spPA$ (cf. \cref{sec:nns_appendix:additional_experiments}).
The improvement in accuracy appears to increase in the large-data regime, where the learning tasks become more~difficult.
Akin to a previously identified scaling trend with size of the pre-training dataset \citep{tamkin2022active}, this suggests a potential scaling trend where the improvement of \itl over random batch selection grows as models are fine-tuned on a larger pool of data.\looseness=-1

\paragraph{Towards task-driven few-shot learning}

Being able to efficiently and automatically select data allows dynamic few-shot fine-tuning to individual tasks~\citep{vinyals2016matching,hardt2024test}, e.g., fine-tuning the model to each test point / query / prompt.
Such task-driven few-shot learning can be seen as a form of ``memory recall'' akin to associative memory~\citep{hopfield1982neural}.
Our results are a first indication that task-driven learning can lead to substantial performance gains, and we believe that this is a promising direction for future studies.\looseness=-1

\section{Related Work}\label{sec:related_work}

\paragraph{Data retrieval}

Substantial prior work has studied data retrieval, e.g., in vector databases, using approximate nearest neighbor search~\citep{johnson2019billion,guo2020accelerating,aumuller2020ann} where cosine similarity is a standard metric.
Following the link between cosine similarity and \ctl, \itl can be seen as a generalization of cosine similarity-based retrieval to batch and query sizes larger than one.
Few-shot fine-tuning can be implemented also for sequences of \emph{different} tasks \citep{vinyals2016matching}, and as such task-driven few-shot fine-tuning can be seen as a form of ``memory recall'' akin to associative memory \citep{hopfield1982neural}.\looseness=-1

\paragraph{Undirected active learning}

Our work extends several previous works in active learning and experimental design \citep{chaloner1995bayesian,settles2009active}.
Uncertainty sampling \citep[US, ][]{lewis1994heterogeneous} is one of the most popular active learning methods.
US selects points $\vx$ with high \emph{prior} uncertainty: ${\vx_{n} = \argmax_{\vx \in \spS} \sigma_{n-1}^2(\vx)}$.
This is in stark contrast to \itl which selects points~$\vx$ that minimize \emph{posterior} uncertainty about $\spA$.
It can be seen that US is the special ``undirected'' (i.e., $\spS \subseteq \spA$) case of \itl when the observation noise is homoscedastic (cf.~\cref{sec:proofs:undirected_itl}).

Several works have previously found entropy-based decision rules to be useful for undirected active learning~\citep{krause2007nonmyopic,guo2007optimistic,krause2008near} and semi-supervised learning~\citep{grandvalet2004semi}.
A variance-based variant of \itl has previously been proposed by \cite{yu2006active} in the special case of global active learning without proving theoretical guarantees.
\cite{shoham2023experimental} recently analyzed experimental designs for global one-shot learning in overparameterized models.\looseness=-1

Substantial work on active learning has studied entropy-based criteria in \emph{parameter-space}, most notably BALD \citep{houlsby2011bayesian,kirsch2019batchbald}, which selects actions ${\vx_n = \argmax_{\vx \in \spX} \I{\vtheta}{y_{\vx}}[\spD_{n-1}]}$, where $\vtheta$ is the random parameter vector of a parametric model (e.g., via Bayesian deep learning).
Such methods are inherently ``undirected'' in the sense that they do not facilitate learning on specific prediction targets.
In contrast, \itl operates in \emph{output-space} where it is straightforward to specify prediction targets, and which is computationally easier.\looseness=-1

\paragraph{Transductive active learning}

In contrast, \itl operates in \emph{output-space} where it is straightforward to specify prediction targets, and which is computationally easier.
Special cases of \itl when ${\spS = \spX}$ and ${\abs{\spA} = 1}$ have been proposed in the foundational work of \cite{mackay1992information} on ``directed'' output-space active learning, and later studied empirically by \cite{wang2021beyond} and \cite{smith2023prediction}.
\cite{kothawade2021similar} recently evaluated \itl empirically on realistic image classification tasks in a pre-training context.
\looseness=-1

\paragraph{Other work on directed active learning}

Special cases of \itl when ${\spS = \spX}$ and ${\abs{\spA} = 1}$ have been proposed in the foundational work of \cite{mackay1992information} on ``directed'' output-space active learning, and later studied empirically by \cite{wang2021beyond} and \cite{smith2023prediction}.
Recently, ``directed'' active learning methods have been proposed for the problem of determining the optimum of an unknown function, also known as best-arm identification \citep{audibert2010best} or pure exploration bandits \citep{bubeck2009pure}.
Entropy search methods \citep{hennig2012entropy,hernandez2014predictive} are widely used and select ${\vx_n = \argmax_{\vx \in \spX} \I{\vx^*}{y_{\vx}}[\spD_{n-1}]}$ in \emph{input-space} where ${\vx^* = \argmax_{\vx} f_{\vx}}$.
Similarly to \itl, \emph{output-space} entropy search methods \citep{hoffman2015output,wang2017max}, which select ${\vx_n = \argmax_{\vx \in \spX} \I{f^*}{y_{\vx}}[\spD_{n-1}]}$ with ${f^* = \max_{\vx} f_{\vx}}$, are more computationally tractable.
In fact, output-space entropy search is a special case of \itl with a stochastic target space \citep{hubotter2024information}.
When subsampling target spaces via Thompson sampling~\citep{thompson1933likelihood,russo2018tutorial}, \itl approximates these output-space methods.
\cite{bogunovic2016truncated} analyze \textsc{TruVar} in the context of Bayesian optimization and level set estimation.
\textsc{TruVar} is akin to a variance-based variant of \itl with a similar notion of ``target space'', but their algorithm and analysis rely on a threshold scheme which requires that ${\spA \subseteq \spS}$.
\cite{fiez2019sequential} introduce the \emph{transductive linear bandit} problem, which is a special case of transductive active learning limited to a linear function class and with the objective of determining the maximum within an initial candidate set.
We mention additional more loosely related works in \cref{sec:related_work_additional}.\looseness=-1

\section{Limitations}\label{sec:limitations}

\paragraph{Better models}

In this work, we focus solely on sequential decision-making \emph{given} some model, rather than asking how one should construct such a model so that it is representative of the ground truth.
Learning stochastic models and approximations thereof have been the subject of much recent work \citep{
blundell2015weight,
maddox2019simple,
daxberger2021laplace,
antoran2022adapting,
lin2023sampling}.
Specifically, \itl benefits if the model $f$ captures the right ``correlations'' between points and if its ``uncertainty'' about the prediction at a specific point is accurate.\footnote{For example, the results from \cref{sec:nns} rely on embeddings being trained to capture the distribution shift.}\looseness=-1

\paragraph{Efficient computation of \itl with large $\spA$}

Naïve computation of the \itl decision rule for $m$ points within~${\spS}$ takes $\BigO{\abs{\spA}^3 + m}$ time and $\BigO{\abs{\spA}^2}$ space~(cf.~\cref{sec:computational_complexity}) which is prohibitive for large $\abs{\spA}$.
A practical solution for large $\spA$ is to subsample $\spA$ (such as in \cref{sec:nns}) within each round, and we provide a formal analysis of this approach for large $\spA$ in \cref{sec:generalizations:roi}.\footnote{In our NN experiments, we subsample $\spA$ to a small set with size equal to the number of classes.}\looseness=-1

\section{Conclusion and Future Work}\label{sec:conclusion}

We framed the active fine-tuning problem as a generalization of active learning, \emph{transductive active learning}, and proposed \itl, an approach which samples adaptively to maximize information gained about specified downstream tasks.
We proved novel generalization bounds which may be of independent interest for active learning.
Finally, we showed that \itl synthesizes prior works on \emph{diversity-oriented} active learning and \emph{relevance-oriented} retrieval, and demonstrated empirically that \itl substantially improves the efficiency of fine-tuning neural networks.\looseness=-1

We believe that the transductive active learning methodology can also improve data-efficiency \emph{beyond supervised learning}, such as in semi-supervised or self-supervised learning and reinforcement learning from human feedback \citep{dwaracherla2024efficient}, and leave this as an exciting direction for future work.
The synthesis of retrieval and learning methods suggests applicability \emph{beyond fine-tuning} such as for in-context learning \citep{brown2020language} and retrieval-augmented generation \citep{lewis2020retrieval} as well as pre-training.\looseness=-1

\section*{Acknowledgements}

Many thanks to Armin Lederer, Johannes Kirschner, Jonas Rothfuss, Lars Lorch, Manish Prajapat, Nicolas~Emmenegger, Parnian Kassraie, and Scott Sussex for their insightful feedback on different versions of this manuscript, as well as Anton Baumann for helpful discussions.

This project was supported in part by the European Research Council (ERC) under the European Union's Horizon 2020 research and Innovation Program Grant agreement no.~815943, the Swiss National Science Foundation under NCCR Automation, grant agreement~51NF40~180545, and by a grant of the Hasler foundation (grant no.~21039).
Jonas Hübotter was supported in part by the German Academic Scholarship Foundation (Studienstiftung).

\bibliography{iclr2024_conference}
\bibliographystyle{iclr2024_conference}

\clearpage\appendix
\section*{\LARGE Appendices}

A general principle of ``transductive learning'' was already formulated by the famous computer scientist Vladimir Vapnik in the 20th century.
Vapnik proposes the following ``imperative for a complex world'':

\begin{chapquote}{\cite{vapnik1982estimation}}
    When solving a problem of interest, do not solve a more general problem as an intermediate step. Try to get the answer that you really need but not a more general one.
\end{chapquote}

These appendices provide additional background, proofs, experiment details, and ablation studies.

\section*{Contents}
\startcontents
\printcontents{}{0}[2]{}
\clearpage

\section{Additional Related Work}\label{sec:related_work_additional}

The general principle of non-active ``transductive learning'' was introduced by \cite{vapnik1982estimation}.
The notion of ``target'' from transductive active learning is akin to the notion of ``task'' in curriculum learning \citep{bengio2009curriculum,graves2017automated,soviany2022curriculum}.
The study of settings where the irreducible uncertainty is zero is related to the study of estimability in experimental design \citep{graybill1961introduction,mutny2022experimental}.
In feature selection, selecting features that maximize information gain with respect to a to-be-predicted label is a standard approach \citep{peng2005feature,vergara2014review,beraha2019feature} which is akin to~\itl.
Balancing relevance and informativeness, similarly to \itl, is also important for data pruning \citep{zheng2023coverage}.
Transductive active learning is complimentary to other learning methodologies, such as semi-supervised learning \citep{gao2020consistency}, self-supervised learning \citep{shwartz2023compress,balestriero2023cookbook}, and meta-learning \citep{kaddour2020probabilistic,rothfuss2023meta}.\looseness=-1

\section{Background}\label{sec:definitions}

\subsection{Information Theory}

Throughout this work, $\log$ denotes the natural logarithm.
Given random vectors $\vx$ and $\vy$, we denote by \begin{align*}
    \H{\vx} &\defeq \E[p(\vx)]{- \log p(\vx)}, \\
    \H{\vx}[\vy] &\defeq \E[p(\vx, \vy)]{- \log p(\vx \mid \vy)}, \quad\text{and} \\
    \I{\vx}{\vy} &\defeq \H{\vx} - \H{\vx}[\vy]
\end{align*} the (differential) entropy, conditional entropy, and information gain, respectively \citep{cover1999elements}.\footnote{One has to be careful to ensure that $\I{\vx}{\vy}$ exists, i.e., $\abs{\I{\vx}{\vy}} < \infty$. We will assume that this is the case throughout this work. When $\vx$ and $\vy$ are jointly Gaussian, this is satisfied when the noise variance $\rho^2$ is positive.}

\subsection{Gaussian Processes}\label{sec:definitions:gps}

The stochastic process $f$ is a Gaussian process (GP, \cite{williams2006gaussian}), denoted ${f \sim \GP{\mu}{k}}$, with mean function $\mu$ and kernel $k$ if for any finite subset ${X = \{\vx_1, \dots, \vx_n\} \subseteq \spX}$, $\vfsub{X} \sim \N{\vmusub{X}}{\mKsub{XX}}$ is jointly Gaussian with mean vector $\vmusub{X}(i) = \mu(\vx_i)$ and covariance matrix $\mKsub{XX}(i, j) = k(\vx_i, \vx_j)$.

In the following, we formalize the assumptions from the GP setting (cf.~\cref{sec:itl:gp_setting}).

\begin{assumption}[Gaussian prior]\label{asm:bayesian_prior}
   We assume that ${f \sim \GP{\mu}{k}}$ with known mean function $\mu$ and kernel $k$.
\end{assumption}

\begin{assumption}[Gaussian noise]\label{asm:bayesian_noise}
   We assume that the noise $\varepsilon_{\vx}$ is mutually independent and zero-mean Gaussian with known variance $\rho^2(\vx) > 0$.
   We write ${\mPsub{X} = \diag{\rho^2(\vx_1), \dots, \rho^2(\vx_n)}}$.
\end{assumption}

Under \cref{asm:bayesian_prior,asm:bayesian_noise}, the posterior distribution of~$f$ after observing points $X$ is $\GP{\mu_n}{k_n}$ with \begin{align*}
    \mu_n(\vx) &= \mu(\vx) + \mKsub{\vx X} \inv{(\mKsub{XX} + \mPsub{X})} (\vysub{X} - \vmusub{X}), \\
    k_n(\vx,\vxp) &= k(\vx,\vxp) - \mKsub{\vx X} \inv{(\mKsub{XX} + \mPsub{X})} \mKsub{X\vxp}, \\
    \sigma_n^2(\vx) &= k_n(\vx,\vx).
\end{align*}
For Gaussian random vectors $\vf$ and $\vy$, the entropy is ${\H{\vf} = \frac{n}{2} \log(2 \pi e) + \frac{1}{2} \log \det{\Var{\vf}}}$, the information gain is ${\I{\vf}{\vy} = \frac{1}{2} \parentheses*{\log\det{\Var{\vy}} - \log\det{\Var{\vy \mid \vf}}}}$, and \begin{align*}
    \gamma_n &= \max_{\substack{X \subseteq \spX \\ |X| \leq n}} \frac{1}{2} \log\det{\mI + \inv{\mPsub{X}} \mKsub{XX}}.
\end{align*}

\section{Provably Efficient Non-adaptive Batch Selection and its Relation to Submodularity}\label{sec:nonmyopic}

Recall the non-adaptive optimization problem \begin{align*}
  B_{n,k} = \argmax_{\substack{B \subseteq \spS \\ \abs{B} = k}} \I{\vfsub{\spA}}{\vysub{B}}[\spD_{n-1}]
\end{align*} from \cref{eq:batch_selection}, and denote by ${B'_{n,k} = \vx_{n,1:k}}$ the greedy approximation from \cref{eq:greedy_batch_selection}.
Note that the objective function \begin{align*}
  F_n(B) &\defeq \I{\vfsub{\spA}}{\vysub{B}}[\spD_{n-1}] = \H{\vfsub{\spA}}[\spD_{n-1}] - \H{\vfsub{\spA}}[\spD_{n-1}, \vysub{B}]
\end{align*} is non-negative and monotone,\footnote{Formally, ${F_n(B) \geq 0}$ and ${F_n(B') \leq F_n(B)}$ for any ${B' \subseteq B \subseteq \spS}$.} since conditional entropy is monotone (which is also called the ``information never hurts'' principle).

Let $\Delta_n(\vx \mid B) \defeq \Delta_n(\{\vx\} \mid B) \defeq F_n(B \cup \{\vx\}) - F(B)$ denote the \emph{marginal gain} of $\vx \in \spS$ given $B \subseteq \spS$ which simplifies to \begin{align*}
  \Delta_n(\vx \mid B) &= \I{\vfsub{\spA}}{\vysub{B}, y_{\vx}}[\spD_{n-1}] - \I{\vfsub{\spA}}{\vysub{B}}[\spD_{n-1}] \\
  &= \H{\vfsub{\spA} \mid \spD_{n-1}, \vysub{B}} - \H{\vfsub{\spA} \mid \spD_{n-1}, \vysub{B}, y_{\vx}} \\
  &= \I{\vfsub{\spA}}{y_{\vx} \mid \spD_{n-1}, \vysub{B}}
\end{align*} and which is precisely the objective function of \itl from \cref{eq:greedy_batch_selection}.

\paragraph{Batch selection via conditional embeddings approximates $B_{n,k}$}

Building upon the theory of maximizing monotone submodular functions \citep{nemhauser1978analysis,krause2014submodular}, \cite{das2018approximate} study greedy maximization under ``approximate'' submodularity:

\begin{definition}[Submodularity ratio of \itl]
  \cite{das2018approximate} define the \emph{submodularity ratio} of $F_n$ up to cardinality $k \geq 1$ as \begin{align}
    \kappa_n(k) \defeq \min_{\substack{B \subseteq B'_{n,k} \\ X \subseteq \spS : |X| \leq k \\ B \cap X = \emptyset}} \frac{\sum_{\vx \in X} \Delta_n(\vx \mid B)}{\Delta_n(X \mid B)}, \label{eq:submodularity_ratio}
  \end{align} where they define $\frac{0}{0} \equiv 1$.
\end{definition}

As a special case of theorem 6 from \cite{das2018approximate}, applying that $F_n$ is non-negative and monotone, we obtain the following result.

\begin{theorem}[Efficiency of batch selection via conditional embeddings]
  For any $n, k \geq 1$, the greedy solution $B'_{n,k}$ provides a $\parentheses*{1-e^{-\kappa_n(k)}}$-approximation of $B_{n,k}$.
\end{theorem}

If $\spS \subseteq \spA$, it is well known \citep[e.g., ][]{srinivas2009gaussian} that $F_n$ is submodular, which implies that $\kappa_n(k) \geq 1$ for all $k \geq 1$.

\section{Proofs}\label{sec:proofs}

\subsection{Definitions}

We write \begin{itemize}
  \item $\sigma^2 \defeq \max_{\vx \in \spX} \sigma_{0}^2(\vx)$, and
  \item $\tilde{\sigma}^2 \defeq \max_{\vx \in \spX} \sigma_{0}^2(\vx) + \rho^2(\vx)$.
\end{itemize}

\subsection{Undirected Case of \itl}\label{sec:proofs:undirected_itl}

We briefly examine the important special case of \itl where $\spS \subseteq \spA$.
In this setting, for all $\vx \in \spS$, the decision rule of \itl simplifies to \begin{align*}
  \Ism{\vfsub{\spA}}{y_{\vx}}[\spD_n] \eeq{1} \Ism{\vfsub{\spA \setminus \{\vx\}}}{y_{\vx}}[f_{\vx}, \spD_n] + \I{f_{\vx}}{y_{\vx}}[\spD_n] \\
  \eeq{2} \Ism{f_{\vx}}{y_{\vx}}[\spD_n] \\
  &= \Hsm{y_{\vx} \mid \spD_n} - \Hsm{\varepsilon_{\vx}}
\end{align*} where \e{1} follows from the chain rule of information gain and $\vx \in \spS \subseteq \spA$; and \e{2} follows from the conditional independence $\vfsub{\spA} \perp y_{\vx} \mid f_{\vx}$.

If additionally $f$ is a GP then \begin{align*}
    \Hsm{y_{\vx} \mid \spD_n} - \H{\varepsilon_{\vx}} = \frac{1}{2} \log \parentheses*{1 + \frac{\Var{f_{\vx} \mid \spD_n}}{\Var{\varepsilon_{\vx}}}}.
\end{align*}
Therefore, when $\spS \subseteq \spA$ and observation noise is homoscedastic, \itl is equivalent to uncertainty sampling.

\subsection{Proof of \cref{thm:variance_convergence}}\label{sec:proofs:variance_convergence}

We will now prove \cref{thm:variance_convergence} where we assume throughout that $\spS \subseteq \spA$.
Let \begin{align}
  \Gamma_n \defeq \max_{\vx \in \spS} \I{\vfsub{\spA}}{y_{\vx}}[\spD_n]
\end{align} denote the objective value of \itl during round $n+1$.

Before proving the convergence outside $\spS$ in \cref{sec:proofs:variance_convergence:outside}, we first prove the convergence of the ``step-wise uncertainty reduction'' $\Gamma_n$.

\begin{theorem}[Bound of uncertainty reduction for \itl]\label{thm:objective_convergence}
  For any $n \geq 1$, if \itl generated the sequence $\{\vx_{i}\}_{i = 1}^{n}$ then \begin{align}
    \Gamma_{n-1} \leq \frac{\gamma_n}{n}.
  \end{align}
\end{theorem}
\begin{proof}
  We have \begin{align*}
    \Gamma_{n-1} &= \frac{1}{n} \sum_{i=0}^{n-1} \Gamma_{n-1} \\
    \eleq{1} \frac{1}{n} \sum_{i=0}^{n-1} \Gamma_i \\
    \eeq{2} \frac{1}{n} \sum_{i=0}^{n-1} \Ism{\vfsub{\spA}}{y_{\vx_{i+1}}}[\spD_{i}] \\
    \eeq{3} \frac{1}{n} \sum_{i=0}^{n-1} \Ism{\vfsub{\spA}}{y_{\vx_{i+1}}}[\vy_{\vx_{1:i}}] \\
    \eeq{4} \frac{1}{n} \Ism{\vfsub{\spA}}{\vysub{\vx_{1:n}}} \\
    &\leq \frac{1}{n} \max_{\substack{X \subseteq \spS \\ |X| = n}} \Ism{\vfsub{\spA}}{\vysub{X}} \\
    \eleq{5} \frac{1}{n} \max_{\substack{X \subseteq \spS \\ |X| = n}} \Ism{\vfsub{X}}{\vysub{X}} \\
    &= \frac{\gamma_n}{n}
  \end{align*} where \e{1} follows from the monotonicity of conditional information gain since $\spS \subseteq \spA$ which implies $\vysub{\vx_{1:i}} \perp y_{\vx} \mid \vfsub{\spA}$ for all $\vx \in \spS$; \e{2} uses the objective of \itl; \e{3} uses that the posterior variance of Gaussians is independent of the realization and only depends on the \emph{location} of observations; \e{4} uses the chain rule of information gain; \e{5} uses ${\vysub{X} \perp \vfsub{\spA} \mid \vfsub{X}}$ and the data processing inequality.
  The conditional independences follow from the assumption that the observation noise is independent.
\end{proof}

We prove a simple consequence of \cref{thm:objective_convergence} for points $\vx$ that are \emph{both} in the target space $\spA$ and in the sample space $\spS$.

\begin{lemma}[Uniform bound of marginal variance within $\spS$]\label{lem:convergence_within_S}
  For any $n \geq 0$ and $\vx \in \spA \cap \spS$, \begin{align}
    \sigma_{n}^2(\vx) \leq 2 \tilde{\sigma}^2 \cdot \Gamma_n.
  \end{align}
\end{lemma}
\begin{proof}
  We have \begin{align*}
    \sigma_{n}^2(\vx) &= \Var{f_{\vx} \mid \spD_n} - \underbrace{\Varsm{f_{\vx} \mid f_{\vx}, \spD_n}}_{0} \\
    \eeq{1} \begin{multlined}[t]
      \Var{y_{\vx} \mid \spD_n} - \rho^2(\vx) \\ - (\Varsm{y_{\vx} \mid f_{\vx}, \spD_n} - \rho^2(\vx))
    \end{multlined} \\
    &= \Var{y_{\vx} \mid \spD_n} - \Varsm{y_{\vx} \mid f_{\vx}, \spD_n} \\
    \eleq{2} \tilde{\sigma}^2 \log\parentheses*{\frac{\Var{y_{\vx} \mid \spD_n}}{\Varsm{y_{\vx} \mid f_{\vx}, \spD_n}}} \\
    &= 2 \tilde{\sigma}^2 \cdot \Ism{f_{\vx}}{y_{\vx}}[\spD_n] \\
    \eleq{3} 2 \tilde{\sigma}^2 \cdot \Ism{\vfsub{\spA}}{y_{\vx}}[\spD_n] \\
    \eleq{4} 2 \tilde{\sigma}^2 \cdot \max_{\vxp \in \spS} \Ism{\vfsub{\spA}}{y_{\vxp}}[\spD_n] \\
    &= 2 \tilde{\sigma}^2 \cdot \Gamma_n
  \end{align*} where \e{1} follows from the noise assumption (cf.~\cref{asm:bayesian_noise}); \e{2} follows from \cref{lem:difference_bound_by_log} and using monotonicity of variance; \e{3} follows from $\vx \in \spA$ and monotonicity of information gain; and \e{4} follows from ${\vx \in \spS}$.
\end{proof}

\subsubsection{Convergence outside $\spS$}\label{sec:proofs:variance_convergence:outside}

We will now show convergence of marginal variance to the irreducible uncertainty for points outside the sample space.

Our proof roughly proceeds as follows:
We construct an ``approximate Markov boundary'' of $\vx$ in $\spS$, and show (1) that the size of this Markov boundary is independent of $n$, and (2) that a small uncertainty reduction within the Markov boundary implies that the marginal variances at the Markov boundary and(!) $\vx$ are small.

\begin{definition}[Approximate Markov boundary]\label{defn:approx_markov_boundary}
  For any $\epsilon > 0$, $n \geq 0$, and $\vx \in \spX$, we denote by $B_{n,\epsilon}(\vx)$ the smallest (multi-)subset of $\spS$ such that \begin{align}
    \Varsm{f_{\vx} \mid \spD_n, \vysub{B_{n,\epsilon}(\vx)}} &\leq \eta_{\spS}^2(\vx) + \epsilon. \label{eq:approx_markov_boundary}
  \end{align}
  We call $B_{n,\epsilon}(\vx)$ an \emph{$\epsilon$-approximate Markov boundary} of $\vx$ in~$\spS$.
\end{definition}

\Cref{eq:approx_markov_boundary} is akin to the notion of the smallest Markov blanket in $\spS$ of some $\vx \in \spX$ (called a \emph{Markov boundary}) which is the smallest set $\spB \subseteq \spS$ such that $f_{\vx} \perp \vfsub{\spS} \mid \vfsub{\spB}$.

\begin{lemma}[Existence of an approximate Markov boundary]\label{lem:approx_markov_boundary}
  For any $\epsilon > 0$, let $k$ be the smallest integer satisfying \begin{align}
    \frac{\gamma_k}{k} \leq \frac{\epsilon \lambda_{\min}^2(\Var{\vfsub{\spS}})}{2 \abs{\spS}^2 \sigma^4 \tilde{\sigma}^2}. \label{eq:markov_boundary_size_condition}
  \end{align}
  Then, for any $n \geq 0$ and $\vx \in \spX$, there exists an $\epsilon$-approximate Markov boundary $B_{n,\epsilon}(\vx)$ of $\vx$ in $\spS$ with size at most $k$.
\end{lemma}
\Cref{lem:approx_markov_boundary} shows that for any $\epsilon > 0$ there exists a universal constant $b_\epsilon$ such that \begin{align}
  \abs{B_{n,\epsilon}(\vx)} \leq b_\epsilon \qquad \forall n \geq 0, \vx \in \spX.
\end{align}
We defer the proof of \cref{lem:approx_markov_boundary} to \cref{sec:missing_proofs:existence_of_an_approximate_markov_boundary} where we also provide an algorithm to compute~$B_{n,\epsilon}(\vx)$.

\begin{lemma}\label{lem:convergence_helper1}
  For any $\epsilon > 0$, $n \geq 0$, and $\vx \in \spX$, \begin{align}
    \sigma_{n}^2(\vx) \leq 2 \sigma^2 \cdot \Ism{f_{\vx}}{\vysub{B_{n,\epsilon}(\vx)}}[\spD_n] + \eta_{\spS}^2(\vx) + \epsilon
  \end{align} where $B_{n,\epsilon}(\vx)$ is an $\epsilon$-approximate Markov boundary of $\vx$ in $\spS$.
\end{lemma}
\begin{proof}
  We have \begin{align*}
    \sigma_{n}^2(\vx) &= \begin{multlined}[t]
      \Var{f_{\vx} \mid \spD_n} - \eta_{\spS}^2(\vx) + \eta_{\spS}^2(\vx)
    \end{multlined} \\
    \eleq{1} \begin{multlined}[t]
      \Var{f_{\vx} \mid \spD_n} - \Varsm{f_{\vx} \mid \vysub{B_{n,\epsilon}(\vx)}, \spD_n} \\ + \eta_{\spS}^2(\vx) + \epsilon
    \end{multlined} \\
    \eleq{2} \begin{multlined}[t]
      \sigma^2 \log\parentheses*{\frac{\Var{f_{\vx} \mid \spD_n}}{\Varsm{f_{\vx} \mid \vysub{B_{n,\epsilon}(\vx)}, \spD_n}}} \\ + \eta_{\spS}^2(\vx) + \epsilon
    \end{multlined} \\
    &= 2 \sigma^2 \cdot \Ism{f_{\vx}}{\vysub{B_{n,\epsilon}(\vx)}}[\spD_n] + \eta_{\spS}^2(\vx) + \epsilon
  \end{align*} where \e{1} follows from the defining property of an $\epsilon$-approximate Markov boundary (cf.~\cref{eq:approx_markov_boundary}); and \e{2} follows from \cref{lem:difference_bound_by_log} and using monotonicity of variance.
\end{proof}

\begin{lemma}\label{lem:convergence_helper2}
  For any $\epsilon > 0, n \geq 0$, and $\vx \in \spA$, \begin{align}
    \Ism{f_{\vx}}{\vysub{B_{n,\epsilon}(\vx)}}[\spD_n] \leq b_{\epsilon} \Gamma_n
  \end{align} where $B_{n,\epsilon}(\vx)$ is an $\epsilon$-approximate Markov boundary of $\vx$ in $\spS$ with $\abs{B_{n,\epsilon}(\vx)} \leq b_\epsilon$.
\end{lemma}
\begin{proof}
  We use the abbreviated notation ${B = B_{n,\epsilon}(\vx)}$.
  We have \begin{align*}
    \Ism{f_{\vx}}{\vysub{B}}[\spD_n] \eleq{1} \Ism{\vfsub{\spA}}{\vysub{B}}[\spD_n] \\
    \eleq{2} \sum_{\tilde{\vx} \in B} \Ism{\vfsub{\spA}}{y_{\tilde{\vx}}}[\spD_n] \\
    \eleq{3} b_{\epsilon} \max_{\tilde{\vx} \in B} \Ism{\vfsub{\spA}}{y_{\tilde{\vx}}}[\spD_n] \\
    \eleq{4} b_{\epsilon} \max_{\tilde{\vx} \in \spS} \Ism{\vfsub{\spA}}{y_{\tilde{\vx}}}[\spD_n] \\
    &= b_{\epsilon} \Gamma_n
  \end{align*} where \e{1} follows from monotonicity of mutual information; \e{2} follows from the chain rule of information gain and the monotonicity of conditional information gain which is due to $\spD_n \perp y_{\tilde{\vx}} \mid \vfsub{\spA}$ for all $\tilde{\vx} \in B \subseteq \spS$; \e{3} follows from $b \leq b_\epsilon$; and \e{4} follows from $B \subseteq \spS$.
\end{proof}

\begin{proof}[Proof of \cref{thm:variance_convergence}]
  Fix any ${\vx \in \spA}$ and ${\epsilon > 0}$.
  By \cref{lem:approx_markov_boundary}, there exists an $\epsilon$-approximate Markov boundary $B_{n,\epsilon}(\vx)$ of $\vx$ in $\spS$ such that $\abs{B_{n,\epsilon}(\vx)} \leq b_\epsilon$.
  We have \begin{align*}
    \sigma_{n}^2(\vx) \eleq{1} 2 \sigma^2 \cdot \Ism{f_{\vx}}{\vysub{B_{n,\epsilon}(\vx)}}[\spD_{n}] + \eta_{\spS}^2(\vx) + \epsilon \\
    \eleq{2} 2 \sigma^2 b_{\epsilon} \Gamma_n + \eta_{\spS}^2(\vx) + \epsilon \\
    \eleq{3} 2 \sigma^2 b_{\epsilon} \frac{\gamma_{n+1}}{{n+1}} + \eta_{\spS}^2(\vx) + \epsilon
  \end{align*} where \e{1} follows from \cref{lem:convergence_helper1}; \e{2} follows from \cref{lem:convergence_helper2}; and follows from \cref{thm:objective_convergence} using that \itl generated the sequence $\{\vx_{i}\}_{i = 1}^{n}$.
  The result follows by setting $C_\epsilon = 2 \sigma^2 b_\epsilon$.
\end{proof}

\subsubsection{Exemplary Application of \cref{thm:variance_convergence}}\label{sec:proofs:variance_convergence:application}

Let $\smash{\epsilon = c \frac{\gamma_{\sqrt{n}}}{\sqrt{n}}}$ with ${c = 2 \abs{\spS}^2 \sigma^4 \tilde{\sigma}^2 / \lambda_{\min}^2(\Var{\vfsub{\spS}})}$.
Then, by \cref{eq:markov_boundary_size_condition}, $b_\epsilon$ can be bounded for instance by $\sqrt{n}$.
Together with \cref{thm:objective_convergence} this implies for \itl that \begin{align*}
  \nu_{n,\epsilon}^2 + \epsilon &\leq 2 \sigma^2 \sqrt{n} \, \Gamma_n + c \gamma_{\sqrt{n}} / \sqrt{n} \\
  &\leq c' \gamma_n / \sqrt{n}
\end{align*} for a constant $c'$, e.g., $c' = 2 \sigma^2 + c$.
This guarantees that the reducible uncertainty of \itl converges, e.g., for Gaussian and smooth Matérn kernels.

\subsubsection{Existence of an Approximate Markov Boundary}\label{sec:missing_proofs:existence_of_an_approximate_markov_boundary}

We now derive \cref{lem:approx_markov_boundary} which shows the existence of an approximate Markov boundary of~$\vx$ in~$\spS$.

\begin{lemma}\label{lem:learn_markov_boundary}
  For any $k \geq 0$, there exists ${B \subseteq \spS}$ with ${|B| = k}$ such that for all ${\vxp \in \spS}$, \begin{align}
    \Var{f_{\vxp} \mid \vysub{B}} \leq 2 \tilde{\sigma}^2 \frac{\gamma_k}{k}.
  \end{align}
\end{lemma}
\begin{proof}
  We choose $B \subseteq \spS$ greedily using the acquisition function \begin{align*}
    \tilde{\vx}_{k} \defeq \argmax_{\tilde{\vx} \in \spS} \Ism{\vfsub{\spS}}{y_{\tilde{\vx}}}[\vysub{B_{k-1}}]
  \end{align*} where $B_k = \tilde{\vx}_{1:k}$.
  Note that this is the special case of \itl with $\spS \subseteq \spA$, and hence, we have \begin{align*}
    \Var{f_{\vxp} \mid \vysub{B_k}} \eleq{1} 2 \tilde{\sigma}^2 \Gamma_k \\
    \eleq{2} 2 \tilde{\sigma}^2 \frac{\gamma_{k+1}}{{k+1}}
  \end{align*} where \e{1} is due to \cref{lem:convergence_within_S}; and \e{2} is due to \cref{thm:objective_convergence} and $\alpha_k(\spS; \spS) \leq 1$.
\end{proof}

\begin{lemma}\label{lem:approx_markov_boundary_property}
  Given any $\epsilon > 0$ and $B \subseteq \spS$, such that for any $\vxp \in \spS$, \begin{align}
    \Var{f_{\vxp} \mid \vysub{B}} \leq \frac{\epsilon \lambda_{\min}^2(\Var{\vfsub{\spS}})}{\abs{\spS}^2 \sigma^4}. \label{eq:approximation_condition}
  \end{align}
  Then for any ${\vx \in \spX}$, \begin{align}
    \Var{f_{\vx} \mid \vysub{B}} \leq \Var{f_{\vx} \mid \vfsub{\spS}} + \epsilon.
  \end{align}
\end{lemma}
\begin{proof}
  We will denote the right-hand side of \cref{eq:approximation_condition} by $\epsilon'$.
  We have \begin{align*}
    &\Var{f_{\vx} \mid \vysub{B}} \\
    \eeq{1} \begin{multlined}[t]
      \E[\vfsub{\spS}]{\Var[f_{\vx}]{f_{\vx} \mid \vfsub{\spS}, \vysub{B}} \mid \vysub{B}} \\ + \Var[\vfsub{\spS}]{\E[f_{\vx}]{f_{\vx} \mid \vfsub{\spS}, \vysub{B}} \mid \vysub{B}}
    \end{multlined} \\
    \eeq{2} \Var[f_{\vx}]{f_{\vx} \mid \vfsub{\spS}, \vysub{B}} + \Var[\vfsub{\spS}]{\E[f_{\vx}]{f_{\vx} \mid \vfsub{\spS}, \vysub{B}} \mid \vysub{B}} \\
    \eeq{3} \underbrace{\Var[f_{\vx}]{f_{\vx} \mid \vfsub{\spS}}}_{\text{irreducible uncertainty}} + \underbrace{\Var[\vfsub{\spS}]{\E[f_{\vx}]{f_{\vx} \mid \vfsub{\spS}} \mid \vysub{B}}}_{\text{reducible (epistemic) uncertainty}}
  \end{align*} where \e{1} follows from the law of total variance; \e{2} uses that the conditional variance of a Gaussian depends only on the location of observations and not on their value; and \e{3} follows from $f_{\vx} \perp \vysub{B} \mid \vfsub{\spS}$ since $B \subseteq \spS$.
  It remains to bound the reducible uncertainty.

  Let $h : \R^d \to \R, \; \vfsub{\spS} \mapsto \E{f_{\vx} \mid \vfsub{\spS}}$ where we write $d \defeq \abs{\spS}$.
  Using the formula for the GP posterior mean, we have \begin{align*}
    h(\vfsub{\spS}) = \E{f_{\vx}} + \transpose{\vz} (\vfsub{\spS} - \E{\vfsub{\spS}})
  \end{align*} where $\vz \defeq \inv{\mK} \vk$, $\mK \defeq \Var{\vfsub{\spS}}$, and $\vk \defeq \Cov{\vfsub{\spS}, f_{\vx}}$.
  Because $h$ is a linear function in $\vfsub{\spS}$ we have for the reducible uncertainty that \begin{align*}
    \Var[\vfsub{\spS}]{h(\vfsub{\spS}) \mid \vysub{B}} &= \transpose{\vz} \Var{\vfsub{\spS} \mid \vysub{B}} \vz \\
    \eleq{1} d \cdot \transpose{\vz} \diag{\Var{\vfsub{\spS} \mid \vysub{B}}} \vz \\
    \eleq{2} \epsilon' d \norm{\vz}_2^2 \\
    &= \epsilon' d \norm{\inv{\mK} \vk}_2^2 \\
    &\leq \epsilon' d \norm{\inv{\mK}}_2^2 \norm{\vk}_2^2 \\
    &= \frac{\epsilon' d \norm{\vk}_2^2}{\lambda_{\min}^2(\mK)}
  \end{align*} where \e{1} follows from \cref{lem:qf_upper_bound_}; and \e{2} follows from the assumption that $\Var{f_{\vxp} \mid \vysub{B}} \leq \epsilon'$ for all $\vxp \in \spS$.
  We have \begin{align*}
    \norm{\vk}_2^2 &= \sum_{\substack{\vxp \in \spS}} \underbrace{\Cov{f_{\vx}, f_{\vxp}}^2}_{\leq \sigma^4} \leq d \sigma^4.
  \end{align*}
  Thus, \begin{align*}
    \Var[\vfsub{\spS}]{h(\vfsub{\spS}) \mid \vysub{B}} \leq \frac{\epsilon' d^2 \sigma^4}{\lambda_{\min}^2(\mK)} = \epsilon.
  \end{align*}
\end{proof}

\begin{proof}[Proof of \cref{lem:approx_markov_boundary}]
  Let $B$ be the set of size $k$ generated by \cref{lem:learn_markov_boundary} to satisfy ${\Var{f_{\vxp} \mid \vysub{B}} \leq 2 \tilde{\sigma}^2 \gamma_k / k}$ for all $\vxp \in \spS$.
  We have for any ${\vx \in \spX}$, \begin{align*}
    \Var{f_{\vx} \mid \spD_n, \vysub{B}} \eleq{1} \Var{f_{\vx} \mid \vysub{B}} \\
    \eleq{2} \Var{f_{\vx} \mid \vfsub{\spS}} + \epsilon
  \end{align*} where \e{1} follows from monotonicity of variance; and \e{2} follows from \cref{lem:approx_markov_boundary_property} and using the condition on $k$.
\end{proof}

We remark that \cref{lem:learn_markov_boundary} provides an algorithm (just ``undirected'' \itl!) to compute an approximate Markov boundary, and the set $B$ returned by this algorithm is a valid approximate Markov boundary for all $\vx \in \spX$.

\subsection{Proof of \cref{thm:width_convergence}}\label{sec:proofs:width_convergence}

We first formalize the assumptions of \cref{thm:width_convergence}:

\begin{assumption}[Regularity of $\opt{f}$]\label{asm:rkhs}
    We assume that $\opt{f}$ is in a reproducing kernel Hilbert space $\spH_k(\spX)$ associated with a kernel $k$ and has bounded norm, that is, $\norm{f}_k \leq B$ for some finite $B \in \R$.
\end{assumption}

\begin{assumption}[Sub-Gaussian noise]\label{asm:noise}
    We further assume that each $\epsilon_{n}$ from the noise sequence $\{\varepsilon_n\}_{n=1}^\infty$ is conditionally zero-mean $\rho(\vx_n)$-sub-Gaussian with known constants ${\rho(\vx) > 0}$ for all $\vx \in \spX$.
    Concretely, \begin{align*}
      \forall n \geq 1, \lambda \in \R : \quad \E{e^{\lambda \epsilon_{n}}}[\spD_{n-1}] \leq \exp\parentheses*{\frac{\lambda^2 \rho^2(\vx_n)}{2}}
    \end{align*} where $\spD_{n-1}$ corresponds to the $\sigma$-algebra generated by the random variables $\{\vx_i,\epsilon_i\}_{i=1}^{n-1}$ and $\vx_n$.
\end{assumption}

We make use of the following foundational result, showing that under the above two assumptions the (misspecified) Gaussian process model from \cref{sec:itl:gp_setting} is an all-time well-calibrated model of $\opt{f}$:

\begin{lemma}[Well-calibrated confidence intervals; \cite{abbasi2013online,chowdhury2017kernelized}]\label{lem:confidence_intervals}
    Pick ${\delta \in (0,1)}$ and let \cref{asm:rkhs,asm:noise} hold.
    Let \begin{align*}
      \beta_{n}(\delta) = \beta_{n}(\delta) = \norm{\opt{f}}_k + \rho \sqrt{2(\gamma_{n} + 1 + \log(1 / \delta))}
    \end{align*} where $\rho = \max_{\vx \in \spX} \rho(\vx)$.\footnote{$\beta_{n}(\delta)$ can be tightened adaptively \citep{emmenegger2023likelihood}.}
    Then, for all $\vx \in \spX$ and $n \geq 0$ jointly with probability at least $1-\delta$, \begin{align*}
        |\opt{f}(\vx) - \mu_{n}(\vx)| \leq \beta_{n}(\delta) \cdot \sigma_{n}(\vx)
    \end{align*} where $\mu_{n}(\vx)$ and $\sigma_{n}^2(\vx)$ are mean and variance (as defined in \cref{sec:definitions:gps}) of the GP posterior of $f(\vx)$ conditional on the observations $\spD_n$, pretending that $\varepsilon_i$ is Gaussian with variance $\rho^2(\vx_i)$.
\end{lemma}

The proof of \cref{thm:width_convergence} is a straightforward application of \cref{lem:confidence_intervals,thm:variance_convergence}:

\begin{proof}[Proof of \cref{thm:width_convergence}]
  By \cref{thm:variance_convergence}, we have that for all ${\vx \in \spA}$, \begin{align*}
    \sigma_n(\vx) &\leq \sqrt{\eta_{\spS}^2(\vx) + \nu_{n,\epsilon^2}^2 + \epsilon^2} \leq \eta_{\spS}(\vx) + \nu_{n,\epsilon^2} + \epsilon.
  \end{align*}
  The result then follows by application of \cref{lem:confidence_intervals}.
\end{proof}

\subsection{Useful Facts and Inequalities}

We denote by $\preceq$ the Loewner partial ordering of symmetric matrices.

\begin{lemma}\label{lem:qf_upper_bound_}
  Let $\mA \in \R^{n \times n}$ be a positive definite matrix with diagonal $\mD$.
  Then, $\mA \preceq n \mD$.
\end{lemma}
\begin{proof}
  Equivalently, one can show $n \mD - \mA \succeq \mzero$.
  We write $\mA \eqdef \msqrt{\mD}\mQ\msqrt{\mD}$, and thus, $\mQ = \mD^{-\nicefrac{1}{2}} \mA \mD^{-\nicefrac{1}{2}}$ is a positive definite symmetric matrix with all diagonal elements equal to $1$.
  It remains to show that \begin{align*}
    n \mD - \mA = \mD^{\nicefrac{1}{2}} (n \mI - \mQ) \mD^{\nicefrac{1}{2}} \succeq \mzero.
  \end{align*}
  Note that $\sum_{i=1}^n \lambda_i(\mQ) = \tr{\mQ} = n$, and hence, all eigenvalues of $\mQ$ belong to $(0,n)$.
\end{proof}

\begin{lemma}\label{lem:difference_bound_by_log}
  If $a, b \in (0,M]$ for some $M > 0$ and $b \geq a$ then \begin{align}
    b - a \leq M \cdot \log\parentheses*{\frac{b}{a}}.
  \end{align}
  If additionally, $a \geq M'$ for some $M' > 0$ then \begin{align}
    b - a \geq M' \cdot \log\parentheses*{\frac{b}{a}}.
  \end{align}
\end{lemma}
\begin{proof}
  Let $f(x) \defeq \log x$.
  By the mean value theorem, there exists $c \in (a,b)$ such that \begin{align*}
    \frac{1}{c} = f'(c) = \frac{f(b) - f(a)}{b - a} = \frac{\log b - \log a}{b - a} = \frac{\log(\frac{b}{a})}{b - a}.
  \end{align*}
  Thus, \begin{align*}
    b - a = c \cdot \log\parentheses*{\frac{b}{a}} < M \cdot \log\parentheses*{\frac{b}{a}}.
  \end{align*}
  Under the additional condition that $a \geq M'$, we obtain \begin{align*}
    b - a = c \cdot \log\parentheses*{\frac{b}{a}} > M' \cdot \log\parentheses*{\frac{b}{a}}.
  \end{align*}
\end{proof}

\section{Correlation-based Transductive Learning}\label{sec:interpretations_approximations:ctl}

We will briefly look at the decision rule \begin{align}
    \vx_{n} = \argmax_{\vx \in \spS} \sum_{\vxp \in \spA} \Cor{f_{\vx},f_{\vxp} \mid \spD_{n-1}}
\end{align} which as we will see can be thought of as loose approximation to \itl
We refer to this decision rule as \ctl, short for \emph{\underline{c}orrelation-based \underline{t}ransductive \underline{l}earning}.

We consider the GP setting, i.e., \cref{asm:bayesian_prior,asm:bayesian_noise} hold.

\paragraph{Approximation of \itl}

The \itl objective can be shown to be lower bounded by \begin{align*}
  \I{\vfsub{\spA}}{y_{\vx}}[\spD_{n-1}] \egeq{1} \frac{1}{|\spA|} \sum_{\vxp \in \spA} \Ism{f_{\vxp}}{y_{\vx}}[\spD_{n-1}] \\
  \eeq{2} - \frac{1}{2 \abs{\spA}} \sum_{\vxp \in \spA} \log\parentheses*{1 - \Cor{f_{\vxp}, y_{\vx} \mid \spD_{n-1}}^2} \\
  \overset{(\rmnum{3})}&{\gtrsim} \sum_{\vxp \in \spA} \Cor{f_{\vxp}, y_{\vx} \mid \spD_{n-1}}^2 + \const
\end{align*} where \e{2} is detailed in example 8.5.1 of \citep{cover1999elements}; and \e{3} follows from Jensen's inequality and exponentiating.
\e{1} is a loose approximation which highlights that \itl takes into account the correlations \emph{between} points in $\spA$ while \ctl does not.

\section{Subsampling Target Spaces}\label{sec:generalizations:roi}

When the target space $\spA$ is large, it may be computationally infeasible to compute the exact objective.
A natural approach to address this issue is to approximate the target space by a smaller set of size $K$.\looseness=-1

One possibility is to select the $K$ ``most explanatory'' points within $\spA$.
This selection problem is similar to the batch selection problem in active learning \citep{holzmuller2023framework} and can be tackled using, e.g., the \emph{greedy max determinant} or \emph{greedy max kernel distance} strategies.
In the remainder of this section, we study an alternative approach which is based on sampling points from $\spA$ according to some probability distribution $\spPA$ supported on $\spA$.\looseness=-1

\subsection{Stochastic Target Spaces}

Concretely, in iteration $n$, a subset $A$ of $K$ points is sampled independently from $\spA$ according to the distribution $\spPA$ and the objective is computed on this subset. Formally, this amounts to a single-sample Monte Carlo approximation of \begin{align}
  \vx_n \in \argmax_{\vx \in \spS} \E[A \iid \spPA]{\I{\vfsub{A}}{y_{\vx}}[\spD_{n-1}]}.
\end{align}

It is straightforward to extend the convergence guarantees to the setting of stochastic target spaces.
Intuitively, after sufficiently many iterations, each $\vx \in \spA$ will have been sampled roughly as often as its ``weight'' (i.e., probability) suggests.\looseness=-1

\begin{lemma}
    Consider the stochastic target space setting.
    For any $n \geq 0, K \geq 1$, and $\vx \in \spA$, with probability at least ${1 - \exp(- n \nu / 8)}$ where $\nu = 1 - (1-\spPA(\vx))^K$, we have that after at most $n$ iterations $\vx$ was within the subsampled target space in at least $n \nu / 2$ iterations.
\end{lemma}
\begin{proof}
  Let $Y_i \sim \mathrm{Binom}(K, \spPA(\vx))$ denote the random variable counting the number of occurrences of $\vx$ in $\spA'_i$.
  Moreover, we write ${X_i \defeq \Ind{\vx \in \spA'_i}}$.
  Note that \begin{align*}
    \nu_i \defeq \E*{X_i} &= \Pr{\vx \in \spA'_i} \\
    &= 1 - \Pr{Y_i = 0} \\
    &= 1 - (1-\spPA(\vx))^K \\
    &= \nu.
  \end{align*}
  Let $X \defeq \sum_{i=1}^n X_i$ with $\E*{X} = n \nu$.
  By Chernoff's bound, \begin{align*}
    \Pr{X \leq \frac{n \nu}{2}} \leq \exp\parentheses*{-\frac{n \nu}{8}}.
  \end{align*}
\end{proof}

With this result, the previously derived convergence guarantees follow also for the stochastic target space setting.
We leave a tighter analysis of stochastic target spaces to future work.\looseness=-1

\section{Computational Complexity}\label{sec:computational_complexity}

Evaluating the acquisition function of \itl in round $n$ requires computing for each $\vx \in \spS$, \begin{align*}
  &\Ism{\vfsub{\spA}}{y_{\vx}}[\spD_n] \\
  &= \frac{1}{2} \log\parentheses*{\frac{\det{\Var{\vfsub{\spA} \mid \spD_n}}}{\det{\Var{\vfsub{\spA} \mid y_{\vx}, \spD_n}}}} &&\text{(forward)} \\
  &= \frac{1}{2} \log\parentheses*{\frac{\Var{y_{\vx} \mid \spD_n}}{\Var{y_{\vx} \mid \vfsub{\spA}, \spD_n}}} &&\text{(backward)}.
\end{align*}
Let $|\spS| = m$ and $|\spA| = k$.
Then, the forward method has complexity $\BigO{m \cdot k^3}$.
For the backward method, observe that the variances are scalar and the covariance matrix $\Var{\vfsub{\spA} \mid \spD_n}$ only has to be inverted once for all points $\vx$. Thus, the backward method has complexity $\BigO{k^3 + m}$.\looseness=-1

When the size $m$ of $\spS$ is relatively small (and hence, all points in $\spS$ can be considered during each iteration of the algorithm), GP inference corresponds simply to computing conditional distributions of a multivariate Gaussian.
The performance can therefore be improved by keeping track of the full posterior distribution over $\vfsub{\spS}$ of size $\BigO{m^2}$ and conditioning on the latest observation during each iteration of the algorithm.
In this case, after each observation the posterior can be updated at a cost of $\BigO{m^2}$ which does not grow with the time $n$, unlike classical GP inference.\looseness=-1

Overall, when $m$ is small, the computational complexity of \itl is $\BigO{k^3 + m^2}$.
When $m$ is large (or possibly infinite) and a subset of $\tilde{m}$ points is considered in a given iteration, the computational complexity of \itl is $\BigO{k^3 + \tilde{m} \cdot n^3}$, neglecting the complexity of selecting the $\tilde{m}$ candidate points.
In the latter case, the computational cost of \itl is dominated by the cost of GP inference.\looseness=-1

\cite{khanna2017scalable} discuss distributed and stochastic approximations of greedy algorithms to (weakly) submodular problems that are also applicable to \itl.\looseness=-1

\section{Additional NN Experiments \& Details}\label{sec:nns_appendix}

We outline the few-shot training of NNs in \cref{alg:itl_nns}.\looseness=-1

\begin{algorithm}[htb]
  \caption{Few-shot training of NNs}
  \label{alg:itl_nns}
\begin{algorithmic}
  \STATE {\bfseries Given:} initialized or pre-trained model $f$, \emph{small} sample $A \sim \spPA$
  \STATE initialize dataset $\spD = \emptyset$
  \STATE \textbf{repeat}
  \STATE\hspace{1em} sample $S \sim \spPS$
  \STATE\hspace{1em} subsample target space $A' \uar A$
  \STATE\hspace{1em} initialize batch $B = \emptyset$
  \STATE\hspace{1em} compute kernel matrix $\mK$ over domain $[S, A']$
  \STATE\hspace{1em} \textbf{repeat $b$ times}
  \STATE\hspace{2em} compute acquisition function w.r.t. $A'$, based on $\mK$
  \STATE\hspace{2em} add maximizer $\vx \in S$ of acquisition function to $B$
  \STATE\hspace{2em} update conditional kernel matrix $\mK$
  \STATE\hspace{1em} obtain labels for $B$ and add to dataset $\spD$
  \STATE\hspace{1em} update $f$ using data $\spD$
\end{algorithmic}
\end{algorithm}

In \cref{sec:nns_appendix:details}, we detail metrics and hyperparameters.
We describe in \cref{sec:nns_appendx:embeddings,sec:nns_appendx:uncertainty_quantification} how to compute the (initial) conditional kernel matrix $\mK$, and in \cref{sec:nns_appendx:batch_selection} how to update this matrix $\mK$ to obtain conditional embeddings for batch selection.\looseness=-1

In \cref{sec:nns_appendix:undirected}, we show that \itl and \ctl significantly outperform a wide selection of commonly used heuristics.
In \cref{sec:nns_appendix:additional_experiments,sec:nns_appendix:ablation_sigma}, we conduct additional experiments and ablations.\looseness=-1

\cite{hubotter2024information} provides additional comparisons with respect to a variance-based trasnductive decision rule.\looseness=-1

\subsection{Experiment Details}\label{sec:nns_appendix:details}

\begin{table}[t]
    \caption{Hyperparameter summary of NN experiments. (*) we train until convergence on oracle validation accuracy.}
    \label{table:nn_hyperparams}
    \vskip 0.15in
    \begin{center}
    \begin{tabular}{@{}lll@{}}
      \toprule
      & MNIST & CIFAR-100 \\
      \midrule
      $\rho$ & $0.01$ & $1$ \\
      $M$ & $30$ & $100$ \\
      $m$ & $3$ & $10$ \\
      $k$ & $1\,000$ & $1\,000$ \\
      batch size $b$ & $1$ & $10$ \\
      \# of epochs & (*) & $5$ \\
      learning rate & $0.001$ & $0.001$ \\
      \bottomrule
    \end{tabular}
    \end{center}
    \vskip -0.1in
\end{table}

We evaluate the accuracy with respect to $\spPA$ using a Monte Carlo approximation with out-of-sample data: \begin{align*}
  \text{accuracy}(\vthetahat) \approx \E*[(\vx, y) \sim \spPA]{\Ind{y = \argmax_i f_i(\vx; \vthetahat)}}.
\end{align*}

We provide an overview of the hyperparameters used in our NN experiments in \cref{table:nn_hyperparams}.
The effect of noise standard deviation $\rho$ is small for all tested $\rho \in [1,100]$ (cf.~ablation study in \cref{table:nn_sigma_ablation}).\footnote{We use a larger noise standard deviation $\rho$ in CIFAR-100 to stabilize the numerics of batch selection via conditional embeddings (cf. \cref{table:nn_sigma_ablation}).}
$M$ denotes the size of the sample ${A \sim \spPA}$.
In each iteration, we select the target space $\spA \gets A'$ as a random subset of $m$ points from $A$.\footnote{This appears to improve the training, likely because it prevents overfitting to peculiarities in the finite sample~$A$~(cf. \cref{fig:nns_subsampled_target_frac}).}
We provide an ablation over $m$ in \cref{sec:nns_appendix:additional_experiments}.\looseness=-1

During each iteration, we select the batch $B$ according to the decision rule from a random sample from $\spPS$ of size $k$.\footnote{In large-scale problems, the work of \cite{coleman2022similarity} suggests to use an (approximate) nearest neighbor search to select the (large) candidate set rather than sampling u.a.r. from $\spPS$. This can be a viable alternative to simply increasing $k$ and suggests future work.}\looseness=-1

Since we train the MNIST model from scratch, we train from random initialization until convergence on oracle validation accuracy.\footnote{That is, to stop training as soon as accuracy on a validation set from $\spPA$ decreases in an epoch.}
We do this to stabilize the learning curves, and provide the least biased (due to the training algorithm) results.
For CIFAR-100, we train for $5$ epochs (starting from the previous iterations' model) which we found to be sufficient to obtain good performance.\looseness=-1

We use the ADAM optimizer \citep{kingma2014adam}.
In our CIFAR-100 experiments, we use a pre-trained EfficientNet-B0 \citep{tan2019efficientnet}, and fine-tune the final and penultimate layers.
We freeze earlier layers to prevent overfitting to the few-shot training data.\looseness=-1

To prevent numerical inaccuracies when computing the \itl objective, we optimize \begin{align}
  \Ism{\vysub{\spA}}{y_{\vx}}[\spD_{n-1}] = \frac{1}{2} \log\parentheses*{\frac{\Var{y_{\vx}}[\spD_{n-1}]}{\Var{y_{\vx}}[\vysub{\spA}, \spD_{n-1}]}}
\end{align} instead of \cref{eq:objective_gp_setting}, which amounts to adding $\rho^2$ to the diagonal of the covariance matrix before inversion.
This appears to improve numerical stability, especially when using gradient embeddings.\footnote{In our experiments, we observe that the effect of various choices of $\rho$ on this slight adaptation of the \itl decision rule has negligible impact on performance. The more prominent effect of $\rho$ appears to arise from the batch selection via conditional embeddings (cf. \cref{table:nn_sigma_ablation}).}\looseness=-1

\subsection{Embeddings and Kernels}\label{sec:nns_appendx:embeddings}

\begin{figure*}[]
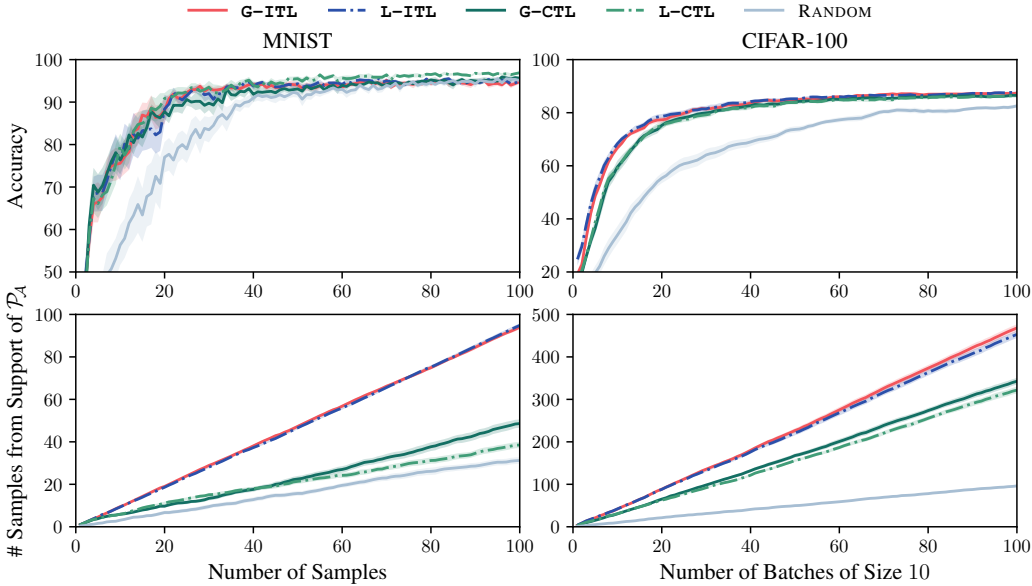

  \incplt[\textwidth]{nns_embeddings}
  \vspace{-0.5cm}
  \caption{Comparison of loss gradient (``G-'') and last-layer embeddings (``L-'').}
  \label{fig:nns_embeddings}
\end{figure*}

Using a neural network to parameterize~$f$, we evaluate the canonical approximations of~$f$ by a stochastic process in the following.\looseness=-1

An embedding~$\vphi(\vx)$ is a latent representation of an input~$\vx$.
Collecting the embeddings as rows in the design matrix~$\mPhi$ of a set of inputs~$X$, one can approximate the network by the linear function~${\vfsub{X} = \mPhi \vbeta}$ with weights~$\vbeta$.
Approximating the weights by ${\vbeta \sim \N{\vmu}{\mSigma}}$ implies that ${\vfsub{X} \sim \N{\mPhi\vmu}{\mPhi \mSigma \transpose{\mPhi}}}$.
The covariance matrix~${\mKsub{XX} = \mPhi \mSigma \transpose{\mPhi}}$ can be succinctly represented in terms of its associated kernel~${k(\vx, \vxp) = \transpose{\vphi(\vx)} \mSigma \vphi(\vxp)}$.
Here, \begin{itemize}[noitemsep]
  \item $\vphi(\vx)$ is the latent representation of $\vx$, and
  \item $\mSigma$ captures the dependencies in the latent space.
\end{itemize}

While any choice of embedding~$\vphi$ is possible, the following are common choices: \begin{enumerate}
  \item \emph{Last-Layer}: A common choice for $\vphi(\vx)$ is the representation of $\vx$ from the penultimate layer of the neural network \citep{holzmuller2023framework}.
  Interpreting the early layers as a feature encoder, this uses the low-dimensional feature map akin to random feature methods \citep{rahimi2007random}.\looseness=-1

  \item \emph{Output Gradients (empirical NTK)}: Another common choice is $\vphi(\vx) = \grad_\vtheta \vf(\vx; \vtheta)$ where $\vtheta$ are the network parameters \citep{holzmuller2023framework}.
  Its associated kernel is known as the (empirical) \emph{neural tangent kernel} (NTK) and the posterior mean of this GP approximates wide NNs trained with gradient descent \citep{arora2019exact,khan2019approximate,he2020bayesian}.
  \cite{kassraie2022neural} derive bounds of $\gamma_n$ under this kernel.
  If $\vtheta$ is restricted to the weights of the final linear layer, then this embedding is simply the last-layer embedding.\looseness=-1

  \item \emph{Loss Gradients}: Another possible choice is \begin{align*}
    \vphi(\vx) = \left. \grad_{\vtheta} \ell(\vf(\vx; \vtheta), \hat{y}(\vx)) \right\rvert_{\vtheta = \vthetahat}
  \end{align*} where $\ell$ is a loss function, $\hat{y}(\vx)$ is the predicted label, and $\vthetahat$ are the current parameter estimates \cite{ash2020deep}.\looseness=-1

  \item \emph{Outputs (empirical NNGP)}: Another possible choice is $\vphi(\vx) = \vf(\vx)$, i.e., the output of the network.
  Its associated kernel is known as the (empirical) \emph{neural network Gaussian process} (NNGP) kernel \citep{lee2017deep}.\looseness=-1
\end{enumerate}

In the additional experiments from this appendix we use last-layer embeddings unless noted otherwise.
We compare the performance of last-layer and the loss gradient embedding \begin{align}
  \vphi(\vx) = \left. \grad_{\vthetap} \ell_{\mathrm{CE}}(\vf(\vx; \vtheta), \hat{y}(\vx)) \right\rvert_{\vtheta = \vthetahat} \label{eq:loss_gradient_embedding}
\end{align} where $\vthetap$ are the parameters of the final output layer, $\smash{\vthetahat}$ are the current parameter estimates, $\smash{\hat{y}(\vx) = \argmax_i f_i(\vx; \vthetahat)}$ are the associated predicted labels, and $\ell_{\mathrm{CE}}$ denotes the cross-entropy loss.
This gradient embedding captures the potential update direction upon observing a new point \citep{ash2020deep}.
Moreover, \cite{ash2020deep} show that for most neural networks, the norm of these gradient embeddings are a conservative lower bound to the norm assumed by taking any other proxy label $\hat{y}(\vx)$.
In \cref{fig:nns_embeddings}, we observe only negligible differences in performance between this and the last-layer embedding.\looseness=-1

\subsection{Towards Uncertainty Quantification in Latent Space}\label{sec:nns_appendx:uncertainty_quantification}

\begin{figure*}[]
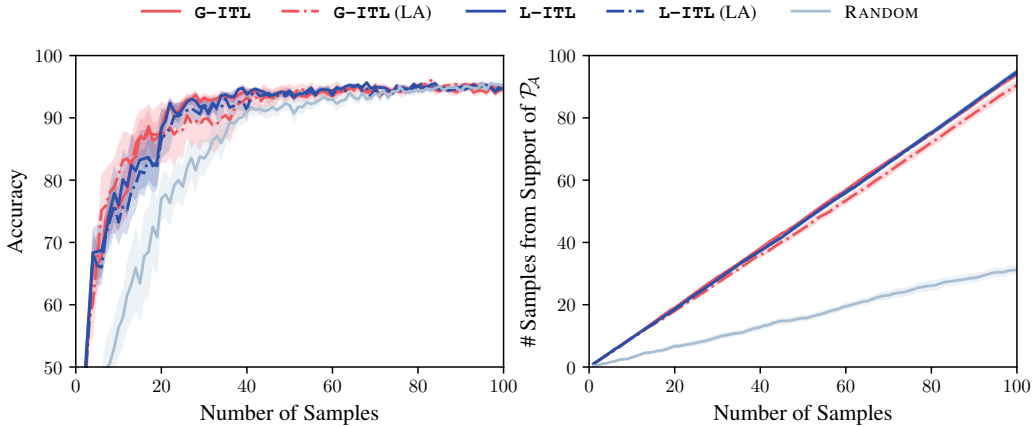

  \incplt[\textwidth]{nns_uncertainty_quantification}
  \vspace{-0.5cm}
  \caption{Uncertainty quantification (i.e., estimation of $\mSigma$) via a Laplace approximation (LA, \cite{daxberger2021laplace}) over last-layer weights using a Kronecker factored log-likelihood Hessian approximation \citep{martens2015optimizing} and the loss gradient embeddings from \cref{eq:loss_gradient_embedding}. The results are shown for the MNIST experiment. We do not observe a performance improvement beyond the trivial approximation $\mSigma = \mI$.}
  \label{fig:nns_uncertainty_quantification}
\end{figure*}

A straightforward and common approximation of the uncertainty about NN weights is given by $\mSigma = \mI$, and we use this approximation throughout our experiments.\looseness=-1

The poor performance of \textsc{UnSa} (cf. \cref{sec:nns_appendix:undirected}) with this approximation suggests that with more sophisticated approximations, the performance of \itl and \ctl can be further improved.
Further research is needed to study the effect of more sophisticated approximations of ``uncertainty'' in the latent space.
For example, with parameter gradient embeddings, the latent space is the network parameter space where various approximations of $\mSigma$ based on Laplace approximation \citep{daxberger2021laplace,antoran2022adapting}, variational inference \citep{blundell2015weight}, or Markov chain Monte Carlo \citep{maddox2019simple} have been studied.
We also evaluate Laplace approximation (LA, \cite{daxberger2021laplace})
for estimating $\mSigma$ but see no improvement (cf. \cref{fig:nns_uncertainty_quantification}).
Nevertheless, we believe that uncertainty quantification is a promising direction for future work, with the potential to improve performance of \itl and its variations substantially.\looseness=-1

\subsection{Batch Selection via Conditional Embeddings}\label{sec:nns_appendx:batch_selection}

We will refer to \cref{eq:greedy_batch_selection} as \textsc{BaCE}, short for \emph{\underline{ba}tch selection via \underline{c}onditional \underline{e}mbeddings}.
\textsc{BaCE} can be implemented efficiently using the Gaussian approximation of $\vfsub{X}$ from \cref{sec:nns_appendx:embeddings} by iteratively conditioning on the previously selected points $\vx_{n,1:i-1}$, and updating the kernel matrix $\mKsub{XX}$ using the closed-form formula for the variance of conditional Gaussians: \begin{align}
  \mKsub{XX} \gets \mKsub{XX} - \frac{1}{\mKsub{\vx_j \vx_j} + \rho^2} \mKsub{X \vx_j} \mKsub{\vx_j X}
\end{align} where $j$ denotes the index of the selected $\vx_{n,i}$ within $X$ and $\rho^2$ is the noise variance.
Note that $\mKsub{\vx_j \vx_j}$ is a scalar and $\mKsub{X \vx_j}$ is a row vector, and hence, this iterative update can be implemented efficiently.\looseness=-1

We remark that \cref{eq:batch_selection,eq:greedy_batch_selection} are natural extensions of previous non-adaptive active learning methods, which typically maximize some notion of ``distance'' between points in the batch, to the ``directed'' setting \citep{ash2020deep,zanette2021design,holzmuller2023framework,pacchiano2024experiment}.
\textsc{BaCE} simultaneously maximizes ``distance'' between points in a batch and minimizes ``distance'' to points in~$\spA$.\looseness=-1

\paragraph{Computational complexity of \textsc{BaCE}}

As derived in \cref{sec:computational_complexity}, a single batch selection step of \textsc{BaCE} has complexity $\BigO{b (k^3 + m^2)}$ where $b$ is the size of the batch, $k = \abs{\spA}$ is the size of the target space, and $m = \abs{\spS}$ is the size of the candidate set.
In the case of large $m$, an alternative implementation whose runtime does not depend on $m$ is described in \cref{sec:computational_complexity}.\looseness=-1

\subsection{Baselines}\label{sec:nns_appendix:undirected}

\begin{figure*}[]
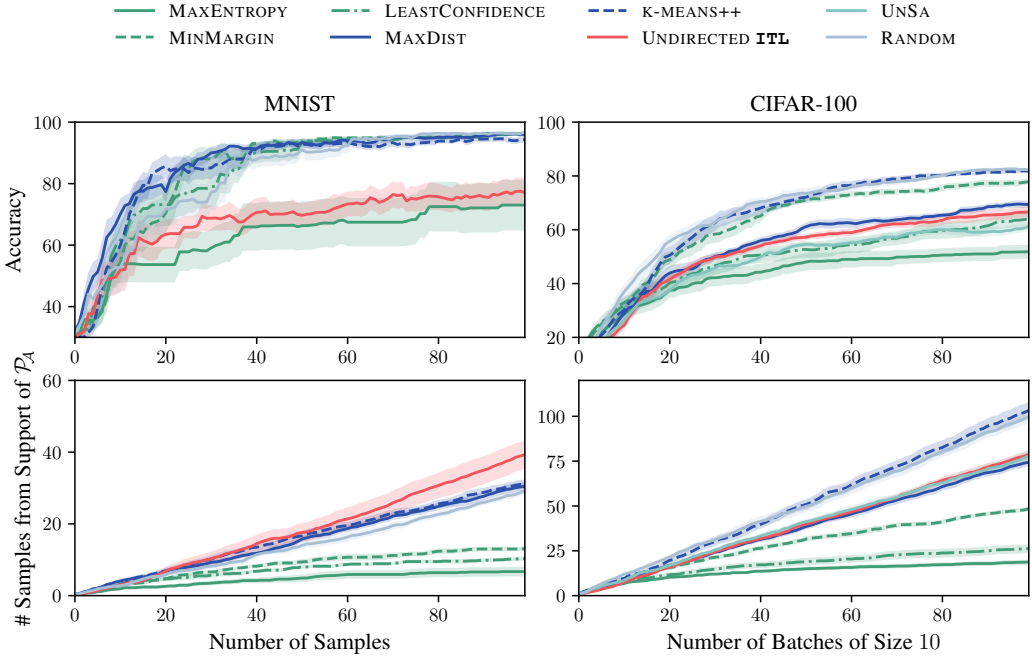

  \incplt[\textwidth]{nns_undirected}
  \vspace{-0.5cm}
  \caption{Comparison of ``undirected'' baselines for the experiment of \cref{fig:nns}. In the MNIST experiment, \textsc{UnSa} and \textsc{Undirected \itl} coincide, and we therefore only plot the latter.}
  \label{fig:nns_undirected}
\end{figure*}

In the following, we briefly describe the most commonly used ``undirected'' decision rules.\looseness=-1

Denote the softmax distribution over labels $i$ at inputs $\vx$ by \begin{align*}
  p_i(\vx; \vthetahat) \propto \exp(f_i(\vx; \vthetahat)).
\end{align*}
The following heuristics computed based on the softmax distribution aim to quantify the ``uncertainty'' about a particular input $\vx$: \begin{itemize}
  \item \textsc{MaxEntropy} \citep{settles2008analysis}: \begin{align*}
    \vx_{n} = \argmax_{\vx \in \spS} \Hsm{p(\vx; \vthetahat_{n-1})}.
  \end{align*}

  \item \textsc{MaxMargin} \citep{scheffer2001active,settles2008analysis}: \begin{align*}
    \vx_{n} = \argmin_{\vx \in \spS} p_1(\vx; \vthetahat_{n-1}) - p_2(\vx; \vthetahat_{n-1})
  \end{align*} where $p_1$ and $p_2$ are the two largest class probabilities.

  \item \textsc{LeastConfidence} \citep{lewis1994sequential,settles2008analysis,hendrycks2017baseline,tamkin2022active}: \begin{align*}
    \vx_{n} = \argmin_{\vx \in \spS} p_1(\vx; \vthetahat_{n-1})
  \end{align*} where $p_1$ is the largest class probability.
\end{itemize}

An alternative class of decision rules aims to select diverse batches by maximizing the distances between points.
Embeddings $\vphi(\vx)$ induce the (Euclidean) embedding distance \begin{align*}
  d_{\vphi}(\vx, \vxp) \defeq \norm{\vphi(\vx) - \vphi(\vxp)}_2.
\end{align*}
Similarly, a kernel $k$ induces the kernel distance \begin{align*}
  d_k(\vx, \vxp) \defeq \sqrt{k(\vx, \vx) + k(\vxp, \vxp) - 2 k(\vx, \vxp)}.
\end{align*}
It is straightforward to see that if $k(\vx, \vxp) = \transpose{\vphi(\vx)} \vphi(\vxp)$, then embedding and kernel distances coincide, i.e., $d_{\vphi}(\vx, \vxp) = d_k(\vx, \vxp)$.\looseness=-1

\begin{itemize}
  \item \textsc{MaxDist} \citep{holzmuller2023framework,yu2010passive,sener2017active,geifman2017deep} constructs the batch by choosing the point with the maximum distance to the nearest previously selected point: \begin{align*}
    \vx_n = \argmax_{\vx \in \spS} \min_{i < n} d(\vx, \vx_i)
  \end{align*}

  \item Similarly, \textsc{k-means++} \citep{holzmuller2023framework} selects the batch via \textsc{k-means++} seeding \citep{arthur2007k,ostrovsky2013effectiveness}.
  That is, the first centroid $\vx_1$ is chosen uniformly at random and the subsequent centroids are chosen with a probability proportional to the square of the distance to the nearest previously selected centroid: \begin{align*}
    \Pr{\vx_n = \vx} \propto \min_{i < n} d(\vx, \vx_i)^2.
  \end{align*}
  When using the loss gradient embeddings from \cref{eq:loss_gradient_embedding}, this decision rule is known as \textsc{BADGE} \citep{ash2020deep}.
\end{itemize}

Finally, we summarize common kernel-based decision rules.
\begin{itemize}
  \item \textsc{Undirected \itl} chooses \begin{align*}
    \vx_n &= \argmax_{\vx \in \spS} \I{\vfsub{\spS}}{y_{\vx}}[\spD_{n-1}] \\
    &= \argmax_{\vx \in \spS} \I{f_{\vx}}{y_{\vx}}[\spD_{n-1}].
  \end{align*}
  This can be shown to be equivalent to \textsc{MaxDet} \citep{holzmuller2023framework} which selects \begin{align*}
    \vx_n = \argmax_{\vx \in \spS} \det{\mKsub{\vx} + \sigma^2 \mI}
  \end{align*} where $\mKsub{\vx}$ denotes the kernel matrix over $\vx_{1:n-1} \cup \{\vx\}$, conditioned on observations~$\spD_{n-1}$.

  \item Uncertainty sampling \citep[\textsc{UnSa}, ][]{lewis1994heterogeneous} which with embeddings $\vphi_{n-1}$ after round $n-1$ selects: \begin{align*}
    \vx_{n} = \argmax_{\vx \in \spS} \sigma_{n-1}^2(\vx) = \argmax_{\vx \in \spS} \norm{\vphi_{n-1}(\vx)}_2^2.
  \end{align*}
  With batch size $b = 1$, \textsc{UnSa} coincides with \textsc{Undirected \itl}.
  When evaluated with gradient embeddings, this acquisition function is similar to previously used ``embedding length'' or ``gradient length'' heuristics \citep{settles2008analysis}.
\end{itemize}

We compare to the abovementioned decision rules and summarize the results in \cref{fig:nns_undirected}.
We observe that most ``undirected'' decision rules perform worse (and often significantly so) than \textsc{Random}.
This is likely due to frequently selecting points from the support of $\spPS$ which are not in the support of $\spPA$ since the points are ``adversarial examples'' that the model $\vthetahat$ is not trained to perform well on.
In the case of MNIST, the poor performance can also partially be attributed to the well-known ``cold-start problem'' \citep{gao2020consistency}.\looseness=-1

In \cref{fig:nns}, we also compare to the following ``directed'' decision rules: \begin{itemize}
  \item \textsc{CosineSimilarity} \citep{settles2008analysis} selects $\vx_{n} = \argmax_{\vx \in \spS} \angle_{\vphi_{n-1}}(\vx,\spA)$ where \begin{align*}
    \angle_{\vphi}(\vx,\spA) \defeq \frac{1}{\abs{\spA}} \sum_{\vxp \in \spA} \frac{\transpose{\vphi(\vx)} \vphi(\vxp)}{\norm{\vphi(\vx)}_2 \norm{\vphi(\vxp)}_2}.
  \end{align*}

  \item \textsc{InformationDensity} \citep{settles2008analysis} is defined as the multiplicative combination of \textsc{MaxEntropy} and \textsc{CosineSimilarity}: \begin{align*}
    \vx_{n} = \argmax_{\vx \in \spS} \Hsm{p(\vx; \vthetahat_{n-1})} \cdot \parentheses*{\angle_{\vphi_{n-1}}(\vx, \spA)}^\beta
  \end{align*} where $\beta > 0$ controls the relative importance of both terms.
  We set $\beta = 1$ in our experiments.
\end{itemize}

\subsection{Additional experiments}\label{sec:nns_appendix:additional_experiments}

\begin{figure*}[]
  \incplt[\textwidth]{nns_imbalanced_train}
  \vspace{-0.5cm}
  \caption{Imbalanced $\spPS$ experiment.}
  \label{fig:nns_imbalanced_train}
\end{figure*}

\begin{figure*}[]
  \incplt[\textwidth]{nns_imbalanced_test}
  \vspace{-0.5cm}
  \caption{Imbalanced $A \sim \spPA$ experiment.}
  \label{fig:nns_imbalanced_test}
\end{figure*}

We conduct the following additional experiments: \begin{enumerate}
  \item \emph{Imbalanced $\spPS$} (\cref{fig:nns_imbalanced_train}): We artificially remove $80\%$ of the support of $\spPA$ from $\spPS$.
  For example, in case of MNIST, we remove $80\%$ of the images with labels $3$, $6$, and $9$ from $\spPS$.
  This makes the learning task more difficult, as $\spPA$ is less represented in $\spPS$, meaning that the ``targets'' are more sparse.
  The trend of \itl outperforming \ctl which outperforms \textsc{Random} is even more pronounced in this setting.

  \item \emph{Imbalanced $A \sim \spPA$} (\cref{fig:nns_imbalanced_test}): We artificially remove $50\%$ of part of the support of $\spPA$ while generating $A \sim \spPA$ to evaluate the robustness of \itl and \ctl in presence of an imbalanced target space $\spA$.
  Concretely, in case of MNIST, we remove $50\%$ of the images with labels $3$ and $6$ from $A$.
  In case of CIFAR-100, we remove $50\%$ of the images with labels $\{0, \dots, 4\}$ from $A$.
  We still observe the same trends as in the other experiments.

  \item \emph{Choice of $k$} (\cref{fig:nns_vtl}): We evaluate the effect of the number of points $k$ at which the decision rule is evaluated.
  Not surprisingly, we observe that the performance of \itl and \ctl improves with larger $k$.

  \item \emph{Choice of $m$} (\cref{fig:nns_subsampled_target_frac}): Next, we evaluate the choice of $m$, i.e., the size of the target space $\spA$ relative to the number $M$ of candidate points $A \sim \spPA$.
  We write $p = m / M$.
  We generally observe that a larger $p$ leads to better performance (with $p=1$ being the best choice).
  However, it appears that a smaller $p$ can be beneficial with respect to accuracy when a large number of batches are selected.
  We believe that this may be because a smaller $p$ improves the diversity between selected batches.

  \item \emph{Choice of $M$} (\cref{fig:nns_n_init}): Finally, we evaluate the choice of $M$, i.e., the size of $A \sim \spPA$.
  Not surprisingly, we observe that the performance of \itl improves with larger $M$.
\end{enumerate}

\begin{figure*}[]
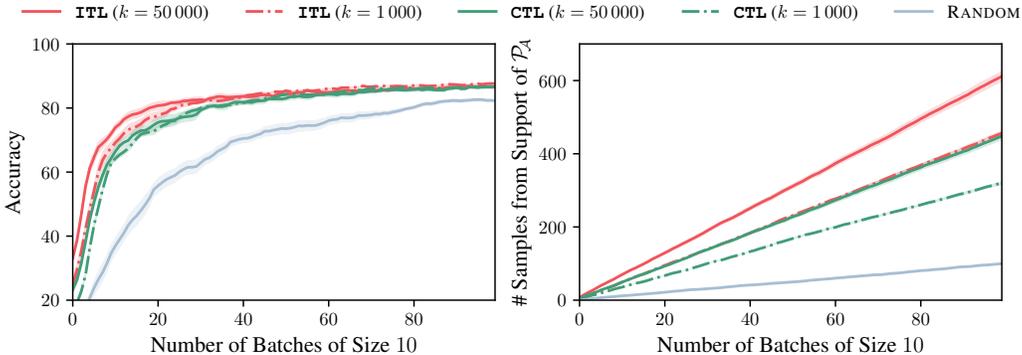

  \incplt[\textwidth]{nns_vtl}
  \vspace{-0.5cm}
  \caption{Choice of $k$ in the CIFAR-100 experiment.}
  \label{fig:nns_vtl}
\end{figure*}

\begin{figure*}[]
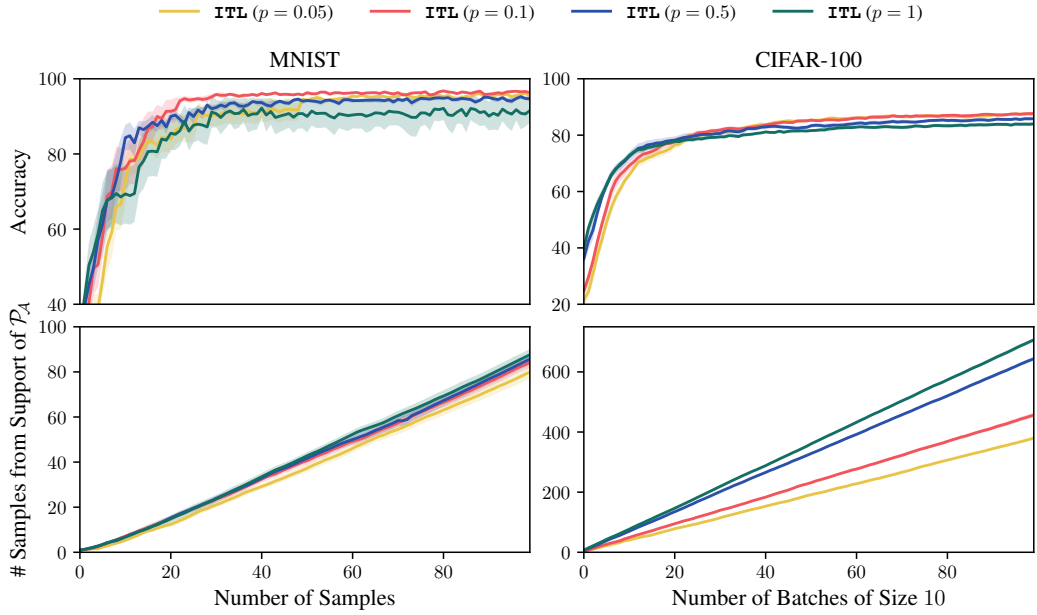

  \incplt[\textwidth]{nns_subsampled_target_frac}
  \vspace{-0.5cm}
  \caption{Evaluation of the choice of $m$ relative to the size $M$ of $A \sim \spPA$. Here, $p = m / M$.}
  \label{fig:nns_subsampled_target_frac}
\end{figure*}

\begin{figure*}[]
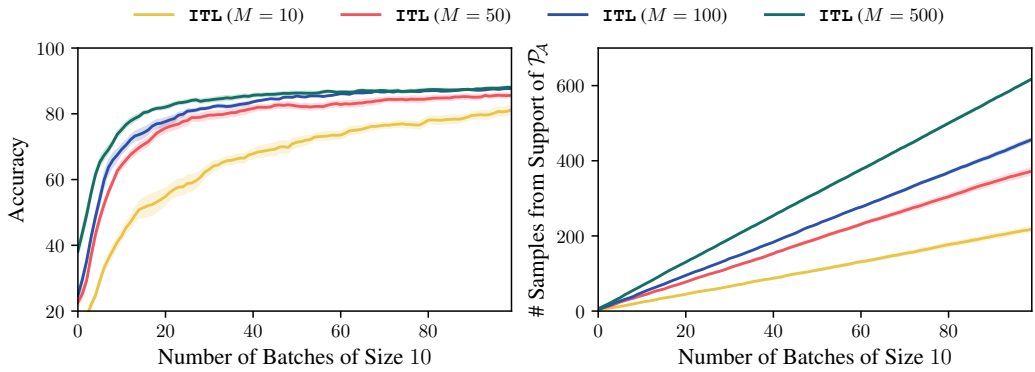

  \incplt[\textwidth]{nns_n_init}
  \vspace{-0.5cm}
  \caption{Evaluation of the choice of $M$, i.e., the size of $A \sim \spPA$, in the CIFAR-100 experiment.}
  \label{fig:nns_n_init}
\end{figure*}

\subsection{Ablation study of noise standard deviation $\rho$}\label{sec:nns_appendix:ablation_sigma}

In \cref{table:nn_sigma_ablation}, we evaluate the CIFAR-100 experiment with different noise standard deviations $\rho$.
We observe that the performance of batch selection via conditional embeddings drops (mostly for the less numerically stable gradient embeddings) if $\rho$ is too small, since this leads to numerical inaccuracies when computing the conditional embeddings.
Apart from this, the effect of $\rho$ is negligible.\looseness=-1

\begin{table*}[]
  \caption{Ablation study of noise standard deviation $\rho$ in the CIFAR-100 experiment. We list the accuracy after $100$ rounds per decision rule, with its standard error over $10$ random seeds. ``(top-$b$)'' denotes variants where batches are selected by taking the top-$b$ points according to the decision rule rather than using batch selection via conditional embeddings. Shown in \textbf{bold} are the best performing decision rules, and shown in \textit{italics} are results due to numerical instability.}
  \label{table:nn_sigma_ablation}
  \vskip 0.15in
  \begin{center}
  \begin{tabular}{@{}lllll@{}}
    \toprule
    $\rho$ & $0.0001$ & $0.01$ & $1$ & $100$ \\
    \midrule
    \gitl & $\mathit{78.26\pm1.40}$ & $\mathit{79.12\pm1.19}$ & $\mathbf{87.16\pm0.29}$ & $\mathbf{87.18\pm0.28}$ \\
    \litl & $\mathbf{87.52\pm0.48}$ & $\mathbf{87.52\pm0.41}$ & $\mathbf{87.53\pm0.35}$ & $86.47\pm0.22$ \\
    \gctl & $\mathit{58.68\pm2.11}$ & $\mathit{81.44\pm1.04}$ & $86.52\pm0.44$ & $\mathbf{86.92\pm0.56}$ \\
    \lctl & $\mathbf{86.40\pm0.71}$ & $\mathbf{86.38\pm0.75}$ & $86.00\pm0.69$ & $84.78\pm0.39$ \\
    \gitl (top-$b$) & $85.84\pm0.54$ & $85.92\pm0.52$ & $85.84\pm0.54$ & $85.55\pm0.46$ \\
    \litl (top-$b$) & $85.44\pm0.58$ & $85.46\pm0.54$ & $85.44\pm0.59$ & $85.29\pm0.36$ \\
    \gctl (top-$b$) & $82.27\pm0.67$ & $82.27\pm0.67$ & $82.27\pm0.67$ & $82.27\pm0.67$ \\
    \lctl (top-$b$) & $80.73\pm0.68$ & $80.73\pm0.68$ & $80.73\pm0.68$ & $80.73\pm0.68$ \\
    \textsc{BADGE} & $83.24\pm0.60$ & $83.24\pm0.60$ & $83.24\pm0.60$ & $83.24\pm0.60$ \\
    \textsc{InformationDensity} & $79.24\pm0.51$ & $79.24\pm0.51$ & $79.24\pm0.51$ & $79.24\pm0.51$ \\
    \textsc{Random} & $82.49\pm0.66$ & $82.49\pm0.66$ & $82.49\pm0.66$ & $82.49\pm0.66$ \\
    \bottomrule
  \end{tabular}
  \end{center}
  \vskip -0.1in
\end{table*}

\begin{table*}[t]
    \caption{Magnitudes of $\gamma_n$ for common kernels. The magnitudes hold under the assumption that $\spX$ is compact. Here, $B_{\nu}$ is the modified Bessel function. We take the magnitudes from Theorem 5 of \cite{srinivas2009gaussian} and Remark 2 of \cite{vakili2021information}. The notation $\BigOTil{\cdot}$ subsumes log-factors. For $\nu = 1/2$, the Matérn kernel is equivalent to the Laplace kernel. For $\nu \to \infty$, the Matérn kernel is equivalent to the Gaussian kernel. The functions sampled from a Matérn kernel are $\lceil\nu\rceil-1$ mean square differentiable.}
    \label{table:gamma_rates}
    \vskip 0.15in
    \begin{center}
    \begin{tabular}{@{}lll@{}}
      \toprule
      Kernel & $k(\vx,\vxp)$ & $\gamma_n$ \\
      \midrule
      Linear & $\transpose{\vx} \vxp$ & $\BigO{d \log(n)}$ \\
      Gaussian & $\exp\parentheses*{-\frac{\norm{\vx - \vxp}_2^2}{2 h^2}}$ & $\BigOTil{\log^{d+1}(n)}$ \\
      Laplace & $\exp\parentheses*{-\frac{\norm{\vx - \vxp}_1}{h}}$ & $\BigOTil{n^{\frac{d}{1 + d}}\log^{\frac{1}{1+d}}(n)}$ \\
      Matérn & $\frac{2^{1-\nu}}{\Gamma(\nu)}\parentheses*{\frac{\sqrt{2\nu}\norm{\vx-\vxp}_2}{h}}^{\nu} B_{\nu} \parentheses*{\frac{\sqrt{2\nu}\norm{\vx-\vxp}_2}{h}}$ & $\BigOTil{n^{\frac{d}{2\nu + d}}\log^{\frac{2\nu}{2\nu+d}}(n)}$ \\
      \bottomrule
    \end{tabular}
    \end{center}
    \vskip -0.1in
\end{table*}

\end{document}